\documentclass{article}

\usepackage{amsmath,amssymb,amsfonts,amsthm}
\usepackage{xspace}
\usepackage{xcolor}
\usepackage{graphicx}
\usepackage{caption}
\usepackage{subcaption}
\usepackage{enumitem}

\usepackage{array}
\usepackage{booktabs}
\usepackage{tabularx}
\usepackage{tabularray}
\UseTblrLibrary{booktabs}
\usepackage{adjustbox}
\usepackage{pdflscape}
\usepackage{threeparttable}
\usepackage{ragged2e}

\usepackage{tkz-graph}
\usetikzlibrary{shapes.geometric}
\usetikzlibrary{positioning,fit}
\usepackage[most]{tcolorbox}
\usepackage{animate}

\usepackage{algorithm}
\usepackage{algpseudocode}

\usepackage[authoryear,round,compress]{natbib}
\usepackage{hyperref}
\hypersetup{colorlinks,linkcolor={blue},citecolor={blue},urlcolor={blue}}
\usepackage[capitalise]{cleveref}
\crefname{lemma}{lemma}{lemmas}

\newtheorem{theorem}{Theorem}
\newtheorem{lemma}[theorem]{Lemma}

\newtheorem{corollary}{Corollary}[theorem]

\definecolor{nyupurple}{HTML}{57068C}

\newcommand{\R}{\mathbb{R}}
\newcommand{\E}{\mathbb{E}}

\newcommand{\B}{\mathcal{B}}
\newcommand{\A}{\mathcal{A}}
\newcommand{\pr}{\mathbb{P}}

\newcommand{\qt}{\quad\quad\quad}

\newcommand{\FClass}{\mathcal{F}\xspace}
\newcommand{\KL}{\mathrm{KL}}
\newcommand{\Reg}{\mathrm{Reg}}
\newcommand{\cE}{\mathcal E}

\newcommand{\inv}{\mathrm{inv}}
\newcommand{\Inv}{\mathcal{I}}

\newcommand{\bX}{\mathbf{X}}
\newcommand{\bx}{\mathbf{x}}
\newcommand{\bA}{\mathbf{A}}
\newcommand{\ba}{\mathbf{a}}

\newcommand{\GTF}{TF\xspace}

\newcommand{\pib}{\pi_b}

\newcommand{\Dval}{\mathcal D_{\mathrm{val}}}

\newcommand{\ADG}{ADIGen\xspace}
\newcommand{\URR}{URR\xspace}

\newcommand{\numberthis}{\addtocounter{equation}{1}\tag{\theequation}}

\newcommand{\listsafesectioning}{%
  \let\svsubsubsection\subsubsection
  \let\svparagraph\paragraph
  \renewcommand\subsubsection[1]{\leavevmode\svsubsubsection{##1}}%
  \renewcommand\paragraph[1]{\leavevmode\svparagraph{##1}}%
}

\newcommand{\tight}{\abovedisplayskip=3pt \belowdisplayskip=3pt \jot=1pt}
\usepackage{authblk}

\newcommand{\titleString}{
Are Good Generators Good Decision-Makers? \\
Policy Learning for General Interventions via Retargeted Counterfactual Generation
}

\title{\titleString}

\author[1,2]{Raphael C. Kim}
\author[1]{Jingsen Zhu}
\author[1]{Ramin Zabih}
\author[2]{Michele Santacatterina}

\affil[1]{Cornell Tech, Cornell University, New York, NY}
\affil[2]{Department of Biostatistics, Department of Population Health, New York University Grossman School of Medicine, New York, NY}

\date{}

\begin{document}

\maketitle

\begin{abstract}
Generative models are increasingly used to support decision-making in complex systems, where interventions may be joint and high-dimensional, and outcomes are high-dimensional. 
However, using generators for these decision-making settings are challenged by three problems. First, they are often trained on noisy logs with limited intervention data. Second, generator learning can be noisy across different environments. Third, generators are not optimized for the decisions they support.
We introduce policy learning via retargeted counterfactual generation, which trains a generator for the decisions it supports in three steps. We (1) learn a doubly-robust, invariant counterfactual generator for high-dimensional interventions and outcomes, (2) conduct policy-learning based on its rollouts, and (3) retarget the generator toward the learned policy's interventions and relearn the policy, so the generator is accurate where decisions are made.
Theoretically, our generator's excess counterfactual risk has a doubly robust product remainder, and retargeting removes the worst-case density ratio between the logged and learned policies from the regret.
We demonstrate the efficacy of our approach across synthetic data, Cell Painting images of SARS-CoV-2-infected cells, and physical video simulations.
\end{abstract}

\clearpage 

\section{Introduction}
Decision-making in complex systems often requires understanding counterfactuals of general, potentially high-dimensional interventions, possibly across multiple interacting units.
With the recent growth in generative modeling, counterfactual generators are increasingly used to support decision-making, from world models of physical systems to virtual cells for understanding cellular perturbations
\citep{pawlowski2020deepscm,desousaribeiro2023highfidelity,wu2024timevarying,wu2025vci},
with non-counterfactual generative modeling approaches spanning image and video generation
\citep{ho2020ddpm,rombach2022ldm,ho2022videodiffusion,peebles2023scalable}.
With a high-quality generator in hand, one can simulate rollouts under different interventions and conduct policy learning.
However, using a generator for policy learning introduces three challenges.

First, learning a counterfactual generator from logged data requires correcting for selection bias and limited overlap. Logged-data policy learning addresses these problems through weighting, doubly robust estimation, and regularization
\citep{dudik2011dr,swaminathan2015crm,levine2020offline}.
These difficulties are exacerbated for high-dimensional interventions, where limited support and density-ratio estimation become increasingly challenging
\citep{kallus2018continuous,sachdeva2020deficient,saito2025potec}.
Automatic Riesz regression instead learns the required representer directly, avoiding explicit density-ratio estimation \citep{chern22riesz}.
A further complication arises when multiple units interact: counterfactual outcomes are then indexed by the joint intervention, and the relevant object may be the joint counterfactual outcome across units
\citep{hudgens2008interference,tchetgen2012causal}.
The clustered-interference literature formalizes such joint interventions and spillovers, but typically targets treatment or policy effects rather than learning the full joint counterfactual outcome distribution.
Most closely for generative modeling, DoubleGen provides doubly robust counterfactual generation from observational data \citep{Luedtke25}.

Second, a generator can fail to generalize across sources or environments.
Invariant prediction and invariant policy learning address this problem by seeking mechanisms that remain stable across environments
\citep{peters16invariant,saengkyongam2023ipl}.

Third, a generator that is accurate on average need not be accurate where an optimized policy acts: optimization can favor interventions whose outcomes are overestimated by the generator
\citep{smith2006optimizer,fujimoto2019bcq,lambert2020objective}.
Behavior-regularized policy learning limits unsupported optimization
\citep{peters2010reps,fujimoto2019bcq}, while policy-aware model-fitting methods adapt dynamics models to the current or learned policy
\citep{hishinuma2021weighted,yang2022ampl,wang2023pdml,ma2023tom}.

In contrast, \ADG combines \textbf{A}utomatic \textbf{D}ebiased and \textbf{I}nvariant counterfactual \textbf{Gen}eration for high-dimensional interventions and joint outcomes with policy learning and retargeting.
Specifically, our contributions address these three challenges directly.

\textbf{1. Debiasing, High-dimensional Modeling of General Interventions, and Joint Outcomes.}
\ADG minimizes a doubly robust estimate of the risk under an interventional law $\nu$, with a product remainder involving the \URR{} and generator errors. For high-dimensional interventions, we learn a universal Riesz Representer (\URR) \citep{chern22riesz}, avoiding explicit inversion of unstable density ratios. \ADG also accommodates joint counterfactual outcomes across interacting units. We establish excess-risk and outcome-model guarantees in \cref{sec:generator} and evaluate these properties in synthetic data, fluorescence microscopy images, and video models in \cref{sec:exp-sim,sec:exp-rx,sec:exp-video}.

\textbf{2. Invariance.}
To learn mechanisms that generalize across environments, we impose an invariance penalty. We establish generalization guarantees across environments in \cref{sec:policy}
. 

\textbf{3. Retarget the Generator, and Recompute the Policy.}
We retarget the generator toward the learned policy while retaining support from the logging policy, so the generator is more accurate where decisions are made, and then recompute the policy. We provide regret guarantees quantifying the gain from retargeting in \cref{sec:policy} and evaluate these predictions in \cref{sec:exp-sim,sec:exp-rx,sec:exp-video}.

\section{Policy Learning under General Interventions}\label{sec:setup}

\paragraph{Data.} We observe i.i.d.\ clusters $Z_i=(\bX_i,\bA_i,Y_i,E_i)\sim P$ of up to $m < \infty$ interacting units where $\bX_i=(X_{i1},\dots,X_{im})$ denote the set of covariates, $\bA_i = (A_{i1},\dots,A_{im})\in\A^m$ for $\A\subseteq\R^{d_A}$ denotes the joint intervention for all units, $E_i \in \mathcal E$ denotes the environment, $Y_i$ denotes the joint outcome, and $g$, with $|g|\le B < \infty$, denotes a utility function that maps the high-dimensional outcome into a lower-dimensional value that we measure optimality on. 
Let the outcome model be $\mu_0(\bx,\ba)=\E[g(Y(\ba)) \mid\bX=\bx]$. In our data, interventions were logged according to an unknown behavior policy $\pib(\ba \mid \bx)$, which may depend on covariates. Units may interact within a cluster, but clusters are independent from one another as in clustered interference settings \citep{tchetgen2012causal}.

\paragraph{Objective.}
Our objective is policy learning under high-dimensions and possibly multiple interacting units. Define the value function as
$$\mbox{ } V(\pi)=\E_{\bX} \int\pi(\ba\mid\bX)\mu_0(\bX,\ba)d\ba$$
We seek the policy optimizer
\begin{equation}\label{eq:objective}
\pi^* \in \arg \max_{\pi\in\Pi} V(\pi)
\end{equation}
over a policy class $\Pi$, where a policy is a mapping from $\pi: \bx\rightarrow \ba$ choosing all units' interventions. 

Policy learning is known to suffer from optimism or extrapolation error when the optimal policy is poorly supported in the logged data \citep{smith2006optimizer,fujimoto2019bcq}. 
To promote faithful decision-making in regions with sufficient support, many adopt a regularized policy objective where the learner is penalized on divergence from the logger \citep{levine2020offline}.
We adopt an approach based on the Kullback--Leibler (KL) divergence
\begin{equation}\label{eq:kl-objective}
J_\beta(\pi)=V(\pi)-\beta \E_{\bX}\KL\big(\pi(\cdot\mid\bX) \| \pi_b(\cdot\mid\bX)\big)
\end{equation}
Its maximizer over all policies is the exponential tilt of the logging policy
\citep[e.g.][]{peters2010reps,wu2019brac,levine2020offline,zhao2024sharp},
\begin{equation}\label{eq:tilt}
\pi^*_\beta(\ba\mid\bx)=\frac{\pi_b(\ba\mid\bx)\,e^{\mu_0(\bx,\ba)/\beta}}{\E_{\ba'\sim\pi_b(\cdot\mid\bx)}\big[e^{\mu_0(\bx,\ba')/\beta}\big]}
\end{equation}
In our setting, this choice comes with the advantages that  the maximizer of $J_\beta$ is an exponential tilt of the logging policy fit over $\Pi$ by a weighted loss, and its bounded density ratio with respect to the logs avoids the unscalable kernel smoothing over the intervention space that deterministic policies would require.

\paragraph{Approach.}
We learn $\pi^*_\beta$ as outlined in \cref{algo:retarget}. The generator (named \ADG) and its training routine (\cref{algo:adigen}) are developed in \cref{sec:generator}.
Our theoretical analysis proceed as follows. 
First, \cref{thm:adigen} bounds the excess risk of \ADG, and \cref{cor:mu-error} translates this guarantee into outcome-model error under the generator's target intervention law.
Next, \cref{thm:regret-tilt} relates outcome-model error to policy regret, while \cref{cor:regret-tilt-gen} specializes this result to a generator and isolates the additional error due to mismatch between the generator's target law and the learned policy.
 \cref{thm:regret-retarg} gives the regret of the retargeted procedure, and \cref{thm:retarget-cmp} quantifies when and by how much retargeting improves over a conditional generator. Finally, \cref{thm:inv} concludes our generalization error in new environments with and without the invariance penalty. 
Theoretical proofs are deferred to \cref{sec:proofs}.

{\tight
\begin{algorithm}[ht!]
\caption{\ADG-PO: Policy Learning with a Retargeted Counterfactual Generator.}\label{algo:retarget}
\small
\begin{algorithmic}[1]
\Require Halves $D_1,D_2$; validation set $\Dval$; grids $\mathcal B$ (temperature $\beta$), $\Lambda$ (mixing $\lambda$); $K$ candidate actions; $L$ rollouts
\Ensure Policies $\hat\pi_1,\hat\pi_2$
\Statex \textit{\# Round 1: base generator and pilot policy}
\For{$j=1,2$}
  \State Fit $\hat\pi_b^{(j)}$ on $D_j$ \Comment{logger}
  \State $\tilde\theta_j\gets\ADG(D_j;\hat\pi_b^{(j)})$ \Comment{\cref{algo:adigen}}
  \For{$\bX_i\in D_j$, $k\in[K]$}
    \State $\ba_{ik}\sim\hat\pi_b^{(j)}(\cdot\mid\bX_i)$
    \State $\tilde Y_{ik\ell}\sim p_{\tilde\theta_j}(\cdot\mid\bX_i,\ba_{ik})$, $\ell\in[L]$
    \State $\tilde\mu_{ik}\gets\tilde\mu_j(\bX_i,\ba_{ik})$ from $L$ rollouts \Comment{pilot mean}
  \EndFor
  \For{$\beta\in\mathcal B$}
    \State $w^\beta_{ik}\gets\frac{\exp(\tilde\mu_{ik}/\beta)}{\sum_{k'}\exp(\tilde\mu_{ik'}/\beta)}$ \Comment{tilt, \cref{eq:tilt}}
    \State $\tilde\omega_{j,\beta}\gets\arg\max_\omega\sum_{i,k}w^\beta_{ik}\log\pi_\omega(\ba_{ik}\mid\bX_i)$
  \EndFor
  \State $\beta_{3-j}\gets\arg\max_{\beta\in\mathcal B}\hat V_{\mathrm{val}}(\pi_{\tilde\omega_{j,\beta}})$
  \State $\tilde\pi_j\gets\pi_{\tilde\omega_{j,\beta_{3-j}}}$ \Comment{pilot policy}
\EndFor
\Statex \textit{\# Round 2: retarget the generator and refit the policy}
\For{$j=1,2$ and $\lambda\in\Lambda$}
  \State $\pi_{j,\lambda}\gets\lambda\tilde\pi_{3-j}+(1-\lambda)\hat\pi_b^{(3-j)}$ \Comment{target law}
  \State $\hat\theta_{j,\lambda}\gets\ADG(D_j;\pi_{j,\lambda})$
  \For{$\bX_i\in D_j$, $k\in[K]$}
    \State $\ba_{ik}\sim\hat\pi_b^{(3-j)}(\cdot\mid\bX_i)$
    \State $\tilde Y_{ik\ell}\sim p_{\hat\theta_{j,\lambda}}(\cdot\mid\bX_i,\ba_{ik})$, $\ell\in[L]$
    \State $\hat\mu_{ik}\gets\hat\mu_j(\bX_i,\ba_{ik})$ from $L$ rollouts
  \EndFor
  \State $w_{ik}\gets\frac{\exp(\hat\mu_{ik}/\beta_j)}{\sum_{k'}\exp(\hat\mu_{ik'}/\beta_j)}$ \Comment{same $\beta_j$ as pilot}
  \State $\hat\omega_{j,\lambda}\gets\arg\max_\omega\sum_{i,k}w_{ik}\log\pi_\omega(\ba_{ik}\mid\bX_i)$
\EndFor
\Statex \textit{\# Select the mixing weight}
\For{$j=1,2$}
  \State $\hat\lambda_j\gets\arg\max_{\lambda\in\Lambda}\hat V_{\mathrm{val}}(\pi_{\hat\omega_{j,\lambda}})$
  \State $\hat\pi_j\gets\pi_{\hat\omega_{j,\hat\lambda_j}}$
\EndFor
\State \Return $\hat\pi_1,\hat\pi_2$
\end{algorithmic}
\end{algorithm}
}
\subsection*{Step 1. Counterfactual Generation: \ADG}\label{sec:generator}

Step 1 of our approach is learning a robust Automatic, Debiased, and Invariant Counterfactual Generation (\ADG).
Within a cluster, each unit's outcome may depend on every unit's intervention, so the relevant intervention is the cluster's full assignment $\ba$, with counterfactual outcome $Y(\ba)$. 
\ADG targets the counterfactual outcome law $P_{\bx,\ba}(\cdot)=\pr\big(Y(\ba)\in\cdot\mid\bX=\bx\big)$ at assignments $\ba$ drawn from an intervention law $\nu(\cdot\mid\bx)$.

We represent this law by a generator. For a parameter $\theta$, the generator is a transport map $\tau(\theta)(\cdot\mid\bx,\ba):\mathcal U\to\mathcal Y$ applied to noise $U\sim P_U$ \citep{Luedtke25}, so its output has law
$p_\theta(\cdot\mid\bx,\ba)=\tau(\theta)(\cdot\mid\bx,\ba)_\sharp P_U$. 
\ADG chooses $\theta$ so that $p_\theta$ approximates $P_{\bx,\ba}$ under $\nu$ using three components: a \URR which is a density ratio learner that generalizes across interventions, a loss term to enforce causal invariance, and doubly-robust learning to promote learning more efficient counterfactual generators.

\textbf{1. Universal Riesz Representer (\URR).} The Riesz representer, or density ratio for any $\nu$, is $\alpha_\nu=\frac{\nu}{\pi_b}$ \citep{chern22riesz}. Rather than estimating $\pi_b$ and taking the inverse for each $\bA$, we use Riesz regression to derive $\hat{\alpha}$ for any $\bA$:
{\tight\begin{equation}\label{eq:rr}
\arg\min_{\alpha\in\FClass_\alpha}\E_n\Big[\alpha(\bA,\bX,E)^2-2\int\alpha(\ba,\bX,E)\nu(d\ba\mid \bX)\Big]
\end{equation}}
By change of measure, $\E\int\alpha(\ba,\bX,E)\nu(d\ba\mid \bX)=\E[\alpha(\bA,\bX,E)\alpha_\nu(\bA,\bX,E)]$, so the population loss is $\E[(\alpha-\alpha_\nu)^2]-\E[\alpha_\nu^2]$, minimized at $\alpha_\nu$.

\textbf{2. Causal invariance.} We learn representations $S_0(\bX)$, $T_0(\bA)$ for which the underlying causal mechanism is the same across environments (\labelcref{ass:invariance}). The generator conditions on $(S(\bX),T(\bA))$ and is trained with a penalty $\lambda_\inv\hat\Inv_n(\theta)$, where $\Theta_\inv\subseteq\Theta$ is the invariant generator class. This helps to learn latent representations that operate efficiently across environments and over higher-dimensional spaces.

\textbf{3. Doubly-robust loss.} Let $\ell(\theta;y,\ba)$ denote the loss of the generator $\theta$ at outcome $y$ and intervention $\ba$. For a pilot generator $\psi$, let
{\tight\[
\zeta_\psi(\theta;\bx,\ba)=\int\ell\big(\theta;\psi(u\mid\bx,\ba),\ba\big)P_U(du)
\]}
be the expected loss of $\theta$ on outcomes drawn from $\psi$ at $(\bx,\ba)$.
With nuisances $\eta=(\alpha,\psi,S,T)$, we propose the following doubly-robust loss
\begin{align*}
&\underbrace{\alpha(\bA,\bX,E)\big[\ell(\theta;Y,\bA)-\zeta_\psi(\theta;\bX,\bA)\big]}_{\text{reweighted residual at the observed }\bA}\\
&\qt +\underbrace{\int\zeta_\psi(\theta;\bX,\ba)\nu(d\ba\mid \bX)}_{\text{plug-in over }\nu} \numberthis \label{eq:drloss}
\end{align*}
As demonstrated in \cref{thm:adigen}, so long as $\alpha$ is sufficiently smooth, we obtain an unbiased loss, and relaxed rates for our generator convergence compared to using a conditional generator.

\begin{algorithm}[ht!]
\caption{\ADG$(D;\nu)$: automatic, debiased, and invariant counterfactual generation. Returns the generator $\tau_{\hat\theta}(\cdot\mid S(\bx),T(\ba))$.}\label{algo:adigen}\label{algo:inv-main}\label{algo:inv-interf}
\small
\algrenewcommand\algorithmiccomment[1]{\hfill{\footnotesize\textit{// #1}}}
\begin{algorithmic}[1]
\Require Clusters $Z=(\bX,\bA,Y,E)$ in folds $Z_n^1,Z_n^2$; target law $\nu$; classes $\Theta,\Psi,\mathcal S,\mathcal T,\FClass_\alpha$; penalty $\Inv$ with weight $\lambda_\inv\ge0$.
\Statex \textit{\# Step 1: nuisances}
\For{$j=1,2$, on $Z_n^j$}
  \State Fit $S^j(\bX)$, $T^j(\bA)$ \Comment{representations}
  \State Fit $\psi^j$ \Comment{pilot generator}
  \State Fit $\hat\alpha^j(\bA,\bX,E)$ by \cref{eq:rr} \Comment{\URR}
\EndFor
\Statex \textit{\# Step 2: doubly robust loss at the target}
\For{$j=1,2$ and each $Z_i\in Z_n^{3-j}$}
  \State $\tilde Y_{ik}\sim\psi^j(\cdot\mid\bX_i,\bA_i)$, $k\in[K]$ \Comment{logged action}
  \State $\hat\zeta_i(\theta)\gets\frac1K\sum_{k}\ell(\theta;\tilde Y_{ik},\bA_i)$
  \State $\tilde\ba_{ik}\sim\nu(\cdot\mid\bX_i)$, $k\in[K]$ \Comment{target actions}
  \State $\tilde Y_{ik}'\sim\psi^j(\cdot\mid\bX_i,\tilde\ba_{ik})$ \Comment{pilot at target}
  \State $\hat\zeta_i'(\theta)\gets\frac1K\sum_{k}\ell(\theta;\tilde Y_{ik}',\tilde\ba_{ik})$ \Comment{plug-in}
  \State $\hat\alpha_i\gets\hat\alpha^j(\bA_i,\bX_i,E_i)$ \Comment{weight}
  \State $L_i(\theta)\gets\hat\alpha_i\big[\ell(\theta;Y_i,\bA_i)-\hat\zeta_i(\theta)\big]+\hat\zeta_i'(\theta)$
\EndFor
\Statex \textit{\# Step 3: estimation}
\State $\hat\theta\gets\arg\min_{\theta\in\Theta}\E_n\big[L_i(\theta)\big]+\lambda_\inv\hat\Inv_n(\theta)$
\State Fit $S,T$ on $D$ \Comment{final representations}
\State \Return $(\hat\theta,S,T)$
\end{algorithmic}
\end{algorithm}
We formalize guarantees on our generator by proving bounds on the excess risk and divergence. We then use these to upper bound policy learning regret. 
Let the counterfactual risk and excess risk at $\nu$ be
{\tight
\begin{align*}
R_\nu(\theta)&=\E_{\bX}\int\E\big[\ell(\theta;Y(\ba),\ba)\mid \bX\big]\nu(d\ba\mid \bX), \\
G_\nu(\theta)&=R_\nu(\theta)-R_\nu(\theta_0),\quad\theta_0 \in \arg\min_{\theta\in\Theta}R_\nu(\theta)
\end{align*}}
Then, the excess risk of \cref{algo:adigen} is given by:
\begin{theorem}[Excess risk of \ADG]\label{thm:adigen}
Under \labelcref{ass:existence}--\labelcref{ass:consistentInv}, \labelcref{ass:overlap} and \labelcref{ass:invariance}, with clusters i.i.d. and $m$ bounded, for a fixed intervention law $\nu$ the estimator of \cref{algo:adigen} satisfies, with probability at least $1-2e^{-s}$, $G_\nu(\hat\theta)\lesssim\mathcal R_n(\nu)$, where
\begin{align*}
    \mathcal R_n(\nu)&=\underbrace{\delta_{n,\inv}^2+\frac{s'}{n}}_{\text{oracle}}+\underbrace{\Big\|1-\frac{\hat\alpha}{\alpha_\nu}\Big\|_\nu^2 d_{\Psi,\inv}(\hat\psi,\Psi_\inv)^2}_{\text{doubly robust}}\\
    &\qt+\underbrace{\lambda_\inv r_{n,\Inv}}_{\text{invariance}}.
\end{align*}
Here $\delta_{n,\inv}$ is the critical radius of the invariant class $\Theta_\inv$, $d_{\Psi,\inv}(\hat\psi,\Psi_\inv)$ is the distance of the pilot $\hat\psi$ from the invariant pilot class, $r_{n,\Inv}$ is the rate at which $\hat\Inv_n$ estimates $\Inv$, and $\|\cdot\|_\nu$ is the $L^2$ norm under $\bX\sim P$ and $\ba\sim\nu(\cdot\mid\bX)$. Further, $s'=s+\log(2+\log_2(B_0\sqrt n))$ with $B_0=C_2(2C_\alpha+1)$, where $C_2$ bounds the loss increments $|\ell(\theta;\cdot)-\ell(\theta_0;\cdot)|$ and $C_\alpha$ bounds the Riesz representer, $\|\alpha_\nu\|_\infty\le C_\alpha$ \labelcref{ass:overlap}.
\end{theorem}
See \cref{pf:adigen} for details and the proof. In short, under sufficient regularity conditions, our generator converges to the true law, with error that decomposes into the oracle rate, a doubly-robust remainder, and an invariance penalty.

Policy learning uses the generator only through its outcome model. We therefore translate the distributional guarantee for \ADG into a bound on the outcome model induced by its rollouts. Let $\mu_{\hat\theta}(\bx,\ba)=\int g\,dp_{\hat\theta}(\cdot\mid\bx,\ba)$ denote the outcome model induced by rollouts of the trained generator $\hat\theta$. The distributional guarantee of \cref{thm:adigen} then gives a bound on $\hat{\mu}$ error:
\begin{corollary}[Outcome-model error from the divergence bound]\label{cor:mu-error}
Suppose the utility is bounded, $|g|\le B$, and condition C1 of \citet{Luedtke25} holds with exponent $b$ and approximation error $\epsilon$. Then, on the event of \cref{thm:adigen}, $\|\mu_{\hat\theta}-\mu_0\|_\nu^2\lesssim\mathcal R_n(\nu)^{b}+\epsilon $, 
where the constants are as in \cref{thm:adigen}.
\end{corollary}
Thus the generator's accuracy under $\nu$ controls the outcome-model error under the same law, with the same oracle, doubly-robust and invariance components as in \cref{thm:adigen}.
The result follows, as
{\tight \begin{align*}
\big(\mu_{\hat\theta}-\mu_0\big)(\bx,\ba)^2 &=\Big(\int g d\big(p_{\hat\theta}(\cdot\mid\bx,\ba)-P_{\bx,\ba}\big)\Big)^2 \\
&\le B^2 \big\|p_{\hat\theta}(\cdot\mid\bx,\ba)-P_{\bx,\ba}\big\|_1^2
\end{align*}}
Then by condition C1 of \citet{Luedtke25} together with \cref{thm:adigen}, the result follows.
\subsection{Steps 2--3. Policy Learning with Retargeted Counterfactual Generation}\label{sec:policy}

Given the debiased counterfactual generator from \cref{sec:generator}, we learn a policy and then retarget the generator at it, following \cref{algo:retarget}. We first estimate a pilot policy $\tilde\pi$ with the KL-regularized objective in \cref{eq:kl-objective}, retrain the generator at an intervention law built from $\tilde\pi$, and learn the policy once more from the retargeted generator. 
In particular, we retrain the generator towards the mixture
{\tight\[
\pi_\lambda=\lambda\tilde\pi+(1-\lambda)\pi_b, \qquad \lambda\in(0,1]
\]}
where $\tilde\pi$ is the pilot estimate of the optimal policy and $\pi_b$ is the logging policy. The mixture places mass at least $\lambda\tilde\pi$ where the pilot acts, and the rest on logged data. Its Riesz representer is $\alpha_\lambda=\frac{\pi_\lambda}{\pi_b}=\lambda\,\frac{\tilde\pi}{\pi_b}+(1-\lambda) $, 
which we learn by \cref{eq:rr} with $\nu=\pi_\lambda$.

In what follows, we bound the regret of the tilting procedure for an arbitrary outcome model $\hat\mu$ (\cref{thm:regret-tilt}). We then specialize the bound to a model induced by a generator (\cref{cor:regret-tilt-gen}), and bound the regret of the retargeting procedure (\cref{thm:regret-retarg}). We conclude by quantify it's gain over a generator that is not retargeted (\cref{thm:retarget-cmp}), and extend our results to generalization bounds across environments (\cref{thm:inv}).

\begin{theorem}[Regret from Tilting]\label{thm:regret-tilt}
Let $\hat\mu$ be any estimate of the mean outcome $\mu_0$, let $\hat\pi_b$ be any estimate of the logging policy $\pi_b$ with $|\log(\hat\pi_b/\pi_b)|\le L$, and let $\hat\pi_\beta\propto\hat\pi_b e^{\hat\mu/\beta}$ be the plug-in tilt solution. Then, for a constant $C_M$ depending only on $B$, $L$ and $\beta$,
{\tight \[
\Reg(\hat\pi_\beta)\le\frac{2C_M}{\beta}\|\mu_0-\hat\mu\|^2_{\hat\pi_\beta}+2C_M \beta\Big\|\log\frac{\hat\pi_b}{\pi_b}\Big\|^2_{\hat\pi_\beta}
\]}
\end{theorem}
We see that (1) the regret depends on the outcome model's error under the learned policy's interventions, rather than the whole policy space; (2) regret is second order in the model's error; and (3) $\beta$ appears in opposing manners, meaning the temperature trades off outcome model error and the fit of the logging policy.

When we move to decision-making with a generator, error must be carried from the generator to the decision: a generator is accurate under its target law $\nu$, the law of interventions its loss averages over, while the regret is incurred under the learned policy $\hat\pi_\beta$. 
To relate the two, we define the \GTF. For intervention laws $\pi$ and $\nu$, define the transfer factor ($\GTF$) from $\nu$ to $\pi$,
{\tight\begin{equation}\label{eq:gtf}
    \rho(\pi \| \nu)=\Big\|\frac{\pi}{\nu}\Big\|_\infty
\end{equation}}
so $\|f\|^2_{\pi}\leq\rho(\pi\|\nu) \|f\|^2_{\nu}$. A change of measure shows,
\[
    \|\hat\mu-\mu_0\|^2_{\hat\pi_\beta}=\E_{\bX}\int(\hat\mu-\mu_0)^2\frac{\hat\pi_\beta}{\nu}d\nu\leq\rho(\hat\pi_\beta\|\nu)\|\hat\mu-\mu_0\|^2_\nu 
\]
Thus the outcome-model error under the policy is at most its error under the training law, inflated by the \GTF, and the inflation grows as the learned policy concentrates on interventions that the training law underweights.
Combining this with the generator's rate $\mathcal R_n(\nu)$ from \cref{thm:adigen} and the outcome-model bound of \cref{cor:mu-error} gives the following regret bound:
\begin{corollary}[Regret from a generator]\label{cor:regret-tilt-gen}
Let $\mu_{\hat\theta}$ be the mean utility of the generator of \cref{algo:adigen} trained at a fixed target law $\nu$, and let $\hat\pi_\beta\propto\hat\pi_b e^{\mu_{\hat\theta}/\beta}$. 
Write $\mathcal G_n(\nu)=C_D\mathcal R_n(\nu)^b+\epsilon_\nu$ for the generator error and $\mathcal P_n(\pi)=\|\log(\frac{\hat\pi_b}{\pi_b})\|^2_{\pi}$ for the propensity error under a policy $\pi$. On the event of \cref{thm:adigen},
{\tight \[
    \Reg(\hat\pi_\beta)\le\frac{4C_MB^2}{\beta}\rho(\hat\pi_\beta\|\nu)\mathcal G_n(\nu)+2C_M\beta\mathcal P_n(\hat\pi_\beta)
\]}
\end{corollary}
The bound has three sources of error. The \emph{transfer error}, $\frac{4C_MB^2}{\beta}\rho(\hat\pi_\beta\|\nu)$, comes from mismatch between the learned policy $\hat\pi_\beta$ and the law $\nu$ the generator was trained at. The \emph{generator error}, $\mathcal G_n(\nu)$, measures how well the generator learns the counterfactual law $\nu$. The \emph{reference error}, $2C_M\beta\mathcal P_n(\hat\pi_\beta)$, is the cost of a poorly estimated logging policy, not involving the generator.

We take a moment to point out what may hurt a conditional or doubly-robust generator based on this bound. 
A conditional generator takes $\nu=\pi_b$ minimizing generator error but having worst-case density ratio $\rho(\hat\pi_\beta\|\pi_b)$ which is large under biased logging. 
A doubly robust generator at an interventional law $\nu$ has a tradeoff: moving $\nu$ away from the logging policy inflates the weights $\alpha_\nu=\frac{\nu}{\pi_b}$ and hence the generator error $\mathcal G_n(\nu)$, while moving it toward the learned policy shrinks the \GTF $\rho(\hat\pi_\beta\|\nu)$. 

\paragraph{Retargeting.} To overcome the worst case of the \GTF, and the optimizer's curse \citep{smith2006optimizer} that comes from targeting the learned policy directly, while still maintaining an interventional law, we propose a retargeting procedure. First, we learn a pilot policy $\tilde\pi$ from a generator, then retrain the generator toward $\pi_\lambda=\lambda\tilde\pi+(1-\lambda)\pi_b$, with each half of the data retargeted at the pilot from the other half. This yields:
\begin{theorem}[Regret of \ADG-PO]\label{thm:regret-retarg}
Assume \labelcref{ass:weights}. For half $j\in\{1,2\}$, let $\pi_{j,\lambda}=\lambda\tilde\pi_{3-j}+(1-\lambda)\hat\pi_b^{(3-j)} $ be its retargeting law, $\beta_j=\hat\beta_{3-j}$ its temperature, $\hat\mu_j=\mu_{\hat\theta_j}$ the outcome model of the generator trained on $D_j$ at $\pi_{j,\lambda}$, $\hat\pi_j\propto\hat\pi_b^{(3-j)}e^{\hat\mu_j/\beta_j}$ its tilt, and $n_j=|D_j|$. Then, for each $j$, on the event of \cref{thm:adigen},
\[
\Reg_{\beta_j}(\hat\pi_j)\le\frac{4C_MB^2}{\beta_j}\rho(\hat\pi_j\|\pi_{j,\lambda})\mathcal G_{n_j}(\pi_{j,\lambda})+2C_M\beta_j\mathcal P_{n_j}(\hat\pi_j)
\]
with $\rho$, $\mathcal G_{n_j}$ and $\mathcal P_{n_j}$ as in \cref{eq:gtf} and \cref{cor:regret-tilt-gen}. Moreover:
\begin{enumerate}[label=(\roman*),leftmargin=*,nosep]
\item \emph{Transfer.} The \GTF satisfies
{\tight\[
\rho(\hat\pi_j\|\pi_{j,\lambda})\le\min\Big\{\frac{e^{2B/\beta_j}}{1-\lambda},\frac{e^{2\|\hat\mu_j-\tilde\mu_{3-j}\|_\infty/\beta_j}}{\lambda}\Big\}
\]}
\item \emph{Weight cost.} The oracle rate in $\mathcal R_{n_j}(\pi_{j,\lambda})$ 
{\tight\[
\delta^2_{n_j,\inv}\le C_9\big(1+\chi^2(\pi_{j,\lambda}\|\pi_b)\big)\bar\delta^2_{n_j,\inv}
\]}
with $\bar\delta_{n_j,\inv}$ the critical radius under unit weights, and $\chi^2(\pi_{j,\lambda}\|\pi_b)\le2\lambda^2\chi^2(\tilde\pi_{3-j}\|\pi_b)+2\chi^2(\hat\pi_b^{(3-j)}\|\pi_b) $. 
\end{enumerate}
\end{theorem}
The proof is in \cref{pf:regret-retarg}.
Retargeting trains the generator where the learned policy acts. The price is a weight cost that is an average, not worst-case, departure from the data, and it is scaled by $\lambda^2$. 

We now quantify the gain of retargeting over the conditional generator and over any other fixed target law.
In the regret bound of \cref{thm:regret-retarg}, only the first term depends on the target law $\nu$. Let
\begin{align*}
    v&=C_9 \bar\delta^2_{n,\inv}, \quad a\geq \mathcal R_n(\nu)-\delta^2_{n,\inv}+\epsilon_\nu
\end{align*}
where $v$ is the part of the generator error that is inflated by the weights, and $a$ bounds the remaining terms uniformly over the laws compared. Taking $b=1$ for readability and dropping the constant $4C_MB^2C_D$, the first term becomes
\[
    \mathcal{B}(\nu)=\frac1\beta\underbrace{\rho(\hat\pi\|\nu)}_{\text{transfer}}\Big[a+\underbrace{\big(1+\chi^2(\nu\|\pi_b)\big)}_{\text{weight cost}}v\Big]
\]
We now demonstrate the gain of retargeting over the conditional generator grows with the lack of overlap and with the error unaffected by reweighting, and that mixture laws minimize $\mathcal B$ up to a factor of two.
\begin{theorem}[Gain and near-optimality of retargeting]\label{thm:retarget-cmp}
    Let $\pi_\lambda=\lambda\hat\pi+(1-\lambda)\pi_b$, $\rho=\rho(\hat\pi\|\pi_b)$ and $\chi^2=\chi^2(\hat\pi\|\pi_b)$.
    \begin{enumerate}[label=(\roman*),leftmargin=*]
        \item \textbf{Gain over the conditional generator.}
        $\inf_{\lambda\in(0,1]} \mathcal B(\pi_\lambda)\le\mathcal B(\pi_b)$, \& if $\lambda^*=\sqrt{\frac{a+v}{\chi^2v}}\le1$,
        \[
        \frac{\mathcal B(\pi_b)}{\inf_{\lambda\in(0,1]}\mathcal B(\pi_\lambda)}\ge\frac{\sqrt\rho}{2}\sqrt{1+\frac av}
        \]
        \item \textbf{Mixtures Suffice.}  Over all target laws $\nu$,
        $\inf_{\lambda\in(0,1]}\mathcal B(\pi_\lambda)\le2\inf_\nu\mathcal B(\nu) $
    \end{enumerate}
\end{theorem}


\paragraph{Invariance Results: Generalization across environments.}
We close with the regret of a policy deployed in an environment $e'$ it was not trained on. Suppose data come from environments $e\in\cE_{\mathrm{tr}}$ with proportions $p_e$, pooled law $P$, and $\cE\supseteq\cE_{\mathrm{tr}}\cup\{e'\}$.
With $S=S_\theta(\bX)$ the representation of \labelcref{ass:invariance}, let $m^{e}(S,\ba)=\E[g(Y(\ba))\mid S,E=e]$, $\bar m$ be its pooled version, and $\Reg^{e'}_\beta$ be the regret in $e'$ with the pooled logging policy. 
Error is governed by the fact that the generator estimates $\bar m$ while the policy is against $m^{e'}$.
\begin{theorem}[Regret in a new environment]\label{thm:inv}
Let $\rho_{\mathrm{gen}}=\rho(P^{e'}\otimes\hat\pi_j\|P\otimes\pi_{j,\lambda})$, $\rho_{\mathrm{ctx}}=\|dP^{e'}_S/dP_S\|_\infty$,
{\tight \[
\mathcal T_j=\frac{8C_{M'}B^2}{\beta_j}\rho_{\mathrm{gen}}\mathcal G_{n_j}(\pi_{j,\lambda})+2C_{M'}\beta_j\rho_{\mathrm{ctx}}\mathcal P_{n_j}(\hat\pi_j)
\]}
with $C_{M'}$ as $C_M$ with $\mu_0$ replaced by $m^{e'}$. On the event of \cref{thm:regret-retarg} under policy retargeting,
{\tight \[
    \Reg^{e'}_{\beta_j}(\hat\pi_j)\leq \mathcal T_j+\frac{4C_{M'}}{\beta_j} \sup_{e,e''}\|m^{e}-m^{e''}\|^2_\infty
\]}
Suppose the target is a mixture of the training environments, $P^{e'}=\sum_e q_eP^e$. Then
{\tight \begin{equation*}
\Reg^{e'}_{\beta_j}(\hat\pi_j)\leq\mathcal T_j+\frac{4C_{M'}}{\beta_j}\rho_j \sum_e p_e\|m^e_{\hat\theta}-\bar m_{\hat\theta}\|^2_{e}
\end{equation*}}
where $\rho(q\|p)=\max_e q_e/p_e$ and $\rho_j=\rho(\hat\pi_j\|\pi_{j,\lambda})\rho(q\|p)$, denotes the combined action and environment transfer factor.
\end{theorem}

\paragraph{Invariance Penalties.}
The penalty trades the disagreement $\sum_e p_e\|m^e_{\hat\theta}-\bar m_{\hat\theta}\|^2_{e}$ against the term $\lambda_\inv r_{n,\Inv}$ in $\mathcal T_j$, leaving $\rho_j$ unchanged. We haven't discussed about how to choose these, so we consider three penalties of increasing strength. An Average Treatment Effect (\textit{ATE}) penalty matches effects averaged over $S$, which can cancel, so it is too weak for enforcing causal invariance. A Conditional Average Treatment Effect (\textit{CATE}) penalty $\sum_e p_e\|\tilde m^e_\theta-\tilde{\bar m}_\theta\|^2_e$, with $\tilde m_\theta$ centered over actions under $\nu$, matches conditional effects. This suffices, at shift cost $\rho_jr_{n,\Inv}^{1/2}/\beta_j$ from the bound above. Finally, a \textit{divergence} penalty, such as the KL between the environment and pooled laws of $Y$ given $(S,\ba)$, bounds the disagreement by $2B^2$ times the penalty (Pinsker) and transfers the whole counterfactual law. However, this incurs a slower $r_{n,\Inv}$, and practically harder loss to work with. In our experiments, we utilize the CATE approach.
\section{Experiments}\label{sec:exp}
Does a generator that looks accurate make good decisions? We test this on synthetic data, Cell Painting images, and video models, through four predictions from \cref{sec:policy}: (P1) By \cref{cor:regret-tilt-gen}, the regret of a conditional generator grows as overlap falls, since its error is inflated by $\rho(\hat\pi_\beta\|\pi_b)$. (P2) By \cref{thm:regret-tilt}, regret follows the generator's error under the learned policy, not its global error. (P3) By \cref{thm:retarget-cmp}, the gain from retargeting grows as overlap falls. We also test (P4), the effect of the invariance penalty (\cref{thm:inv}).

\paragraph{Setup.} We compare a conditional generator, \ADG{}, and retargeted \ADG{}. The first two differ only by the doubly robust correction. All are scored through the same tilt policy with $L=8$ rollouts, and $\beta$, $\lambda$ are chosen by held-out doubly robust value. Training data are confounded logs, and evaluation is in a region where the logger has little support. Each experiment has an overlap dial $\eta$, with larger $\eta$ meaning more overlap; in the synthetic data and PhyWorld experiment, we also vary the number of intervention dimensions and/or objects $m$. We report regret, the model error where the best policy acts and globally, the signed self-evaluation error, which is the generator's value of its own policy minus the true value, and the paired regret gain from retargeting, defined more closely in \cref{app:exp}.

\subsection{Synthetic Data}\label{sec:exp-sim}
We simulate $n=10{,}000$ independent clusters of $m$ units, each taking a two-dimensional action, with covariate $x_j\sim U([-1,1])$. Actions are logged as $a_j\sim\mathcal N(\gamma(x_j,x_j),I)$ with $\gamma=1$, and the mean reward is given by $\mu_0(x,a)=7.5e^{-2\bar d(x,a)}-0.1\bar r(a)$, for $\bar d(x,a)=\frac1m\sum_{j}\|a_j-a^\circ(x_j)\|^2$ the average squared distance of the actions to their targets $a^\circ(x_j)=(-x_j,\,0.5\,x_j)$, and $\bar r(a)=\frac1m\sum_{j}\|a_j\|^2$ the average squared action norm. Outcomes are then simulated as $y=\mu_0(x,a)+\varepsilon$ with $\varepsilon\sim\mathcal N(0,1.5^2)$. 
The task is built so logs miss signal; the reward has a narrow peak at $(-x_j,0.5x_j)$, while the logged actions cluster around $(x_j,x_j)$. Further, since the reward depends on the average distance over all units, overlap shrinks as the intervention dimension $2m$ grows.
See \cref{app:sim} for more details.

\cref{fig:toy_main_misspec} shows as overlap falls, the conditional generator's regret climbs (P1), because a policy searching for high predicted reward goes exactly where the logs are thin. Retargeting removes regret, and its advantage widens as overlap falls and $m$ grows (P3). In particular, it achieves this by trading global accuracy for accuracy at $\pi^*$ (P2).
Further, when we train our retargeting procedure with invariant penalties, we outperform ADIGen without the retargeting and its conditional comparisons \cref{fig:toy_inv} (P4). When the model is correctly specified (\cref{fig:toy_main_correct}), there is no bias to remove, and retargeting is can be worse because reweighting adds variance.

\begin{figure}[ht!]
    \centering
    \includegraphics[width=\linewidth]{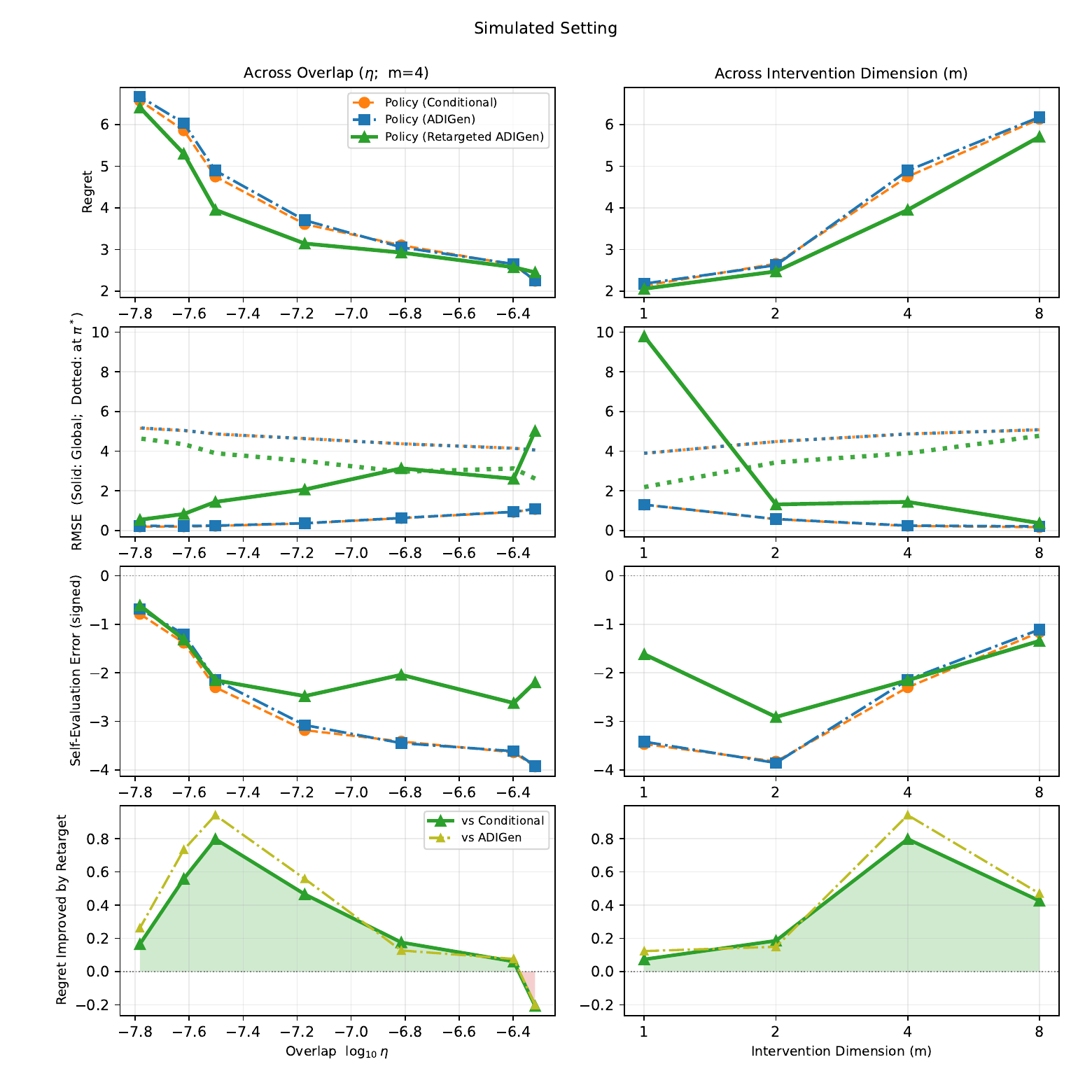}
    \caption{Synthetic data, misspecified outcome model. Left: varying overlap at $m=4$, where $\eta$ is the 99th percentile of $\pi^*/\pi_b$ (larger $\log_{10}\eta$ means less overlap). Right: varying the number of units $m$. Rows, top to bottom: regret; outcome-model RMSE at $\pi^*$ (dotted) and globally (solid); signed self-evaluation error $\hat V(\hat\pi)-V(\hat\pi)$; and regret reduction from retargeting relative to the conditional generator and ADIGen. Retargeted ADIGen attains the lowest regret throughout (P1--P3).}
    \label{fig:toy_main_misspec}
\end{figure}

\subsection{RxRx19a}\label{sec:exp-rx}
RxRx19a \citep{cuccarese2020functional} contains microscopy images of SARS-CoV-2 infected cells treated with $1{,}672$ compounds at six or more concentrations, in two cell lines. We use $32$ compounds screened in both lines. The covariate is the cell line, the environment is the experiment batch, the intervention is the compound and dose ($|\mathcal A|=256$), and the outcome is the well's image. The utility is a rescue score from $0$ (infected) to $1$ (mock-infected). The screen is randomized, so we bias the logs by removing VERO wells at doses $\geq1 \mu$M. In particular, we keep a fraction $\eta\in\{0.1,0.25,0.5,0.75,1\}$ of their training wells, leaving little support where the rescue is high. See \cref{app:exp-rx} for more details.

\begin{figure}[ht!]
    \centering
    \includegraphics[width=0.8\linewidth]{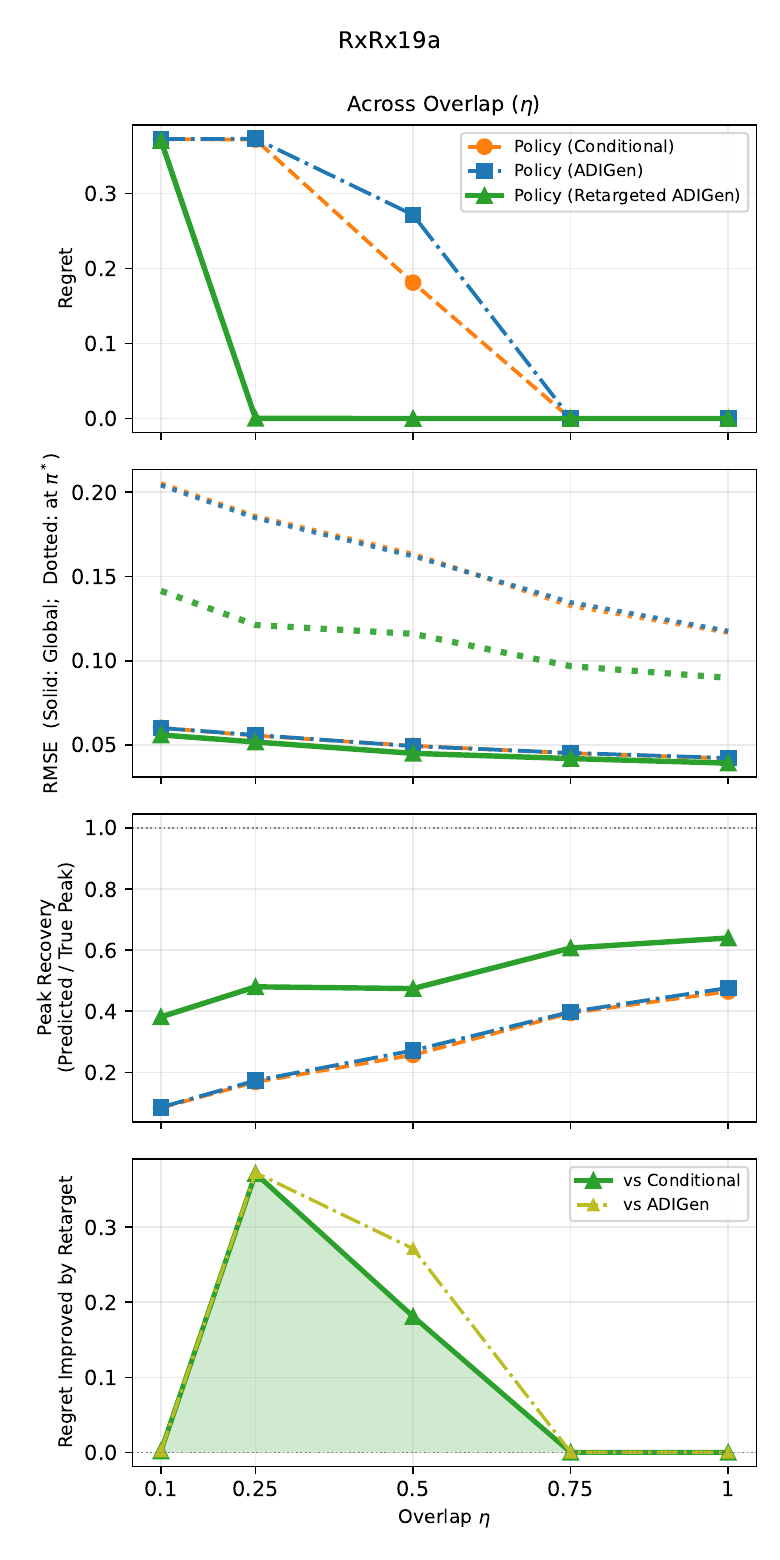}
    \caption{RxRx19a, as overlap $\eta$ (fraction of VERO training wells kept at doses $\geq 1\,\mu$M) varies. Rows, top to bottom: regret of each policy; RMSE of the generator, globally (solid; uniform over the $256$ interventions) and at $\pi^\star$ (dotted; the thinned doses of the three active compounds); peak recovery, the ratio of predicted to true height of the VERO rescue peak; and the regret reduction of retargeted \ADG{} relative to the conditional generator and \ADG{} (P1--P3).}
    \label{fig:rx_rho}
\end{figure}

On real images, the failure of conditional models is easy to see (\cref{fig:rx_rho,fig:rx_inv}). Trained mostly on low doses, the conditional generator predicts a flat dose response, missing high rescue signal, meaning its regret rises as high-dose wells are removed (P1). Interestingly, the retargeted generator recovers the peak of all three active compounds even when only $10\%$ of those wells remain. It is more accurate exactly at the doses the policy selects (P2), and its advantage grows as overlap falls (P3).
When the generator also conditions on the experiment batch, dose (peak) recovery rises, and the retargeted generator stays ahead at every $\eta$ (\cref{app:exp-rx}; P4).

\subsection{Video models}\label{sec:exp-video}
We next test whether similar findings hold when outcomes are videos, using PhyWorld, a two-dimensional physics simulator for evaluating video generation models.
We build a decision task as follows. Outcomes are rendered videos, and the generator is a diffusion transformer (DiT) trained by flow matching in a VAE latent space. Each of $m$ balls, in its own lane, receives a horizontal push and a gust common to all balls. The policy chooses a second push to make per ball so $\ba\in\R^m$. The utility rewards landing each ball near a goal ($-|x_{30}-\text{goal}|$, averaged over balls) and regret is measured by error. 
Scenes with strong gusts ($|g|\in[0.50,0.85]$) lie outside the training range ($|g|\le0.25$), and we observe those gusts with probability $\eta \in \{.002,.005,.01,.05\} $. Details are in \cref{app:exp-video}. 

Varying $\eta$ and the intervention dimension $m\in\{1,2,4,8\}$, we find the same pattern as before: when overlap is low, retargeting has the lowest regret, followed by \ADG and then the conditional model (P1, P3). The gain holds at every $m$, and the retargeted generator has the smallest self-evaluation error, so its accuracy is concentrated where its policy acts (P2).
\begin{figure}[ht!]
    \centering
    \includegraphics[width=\linewidth]{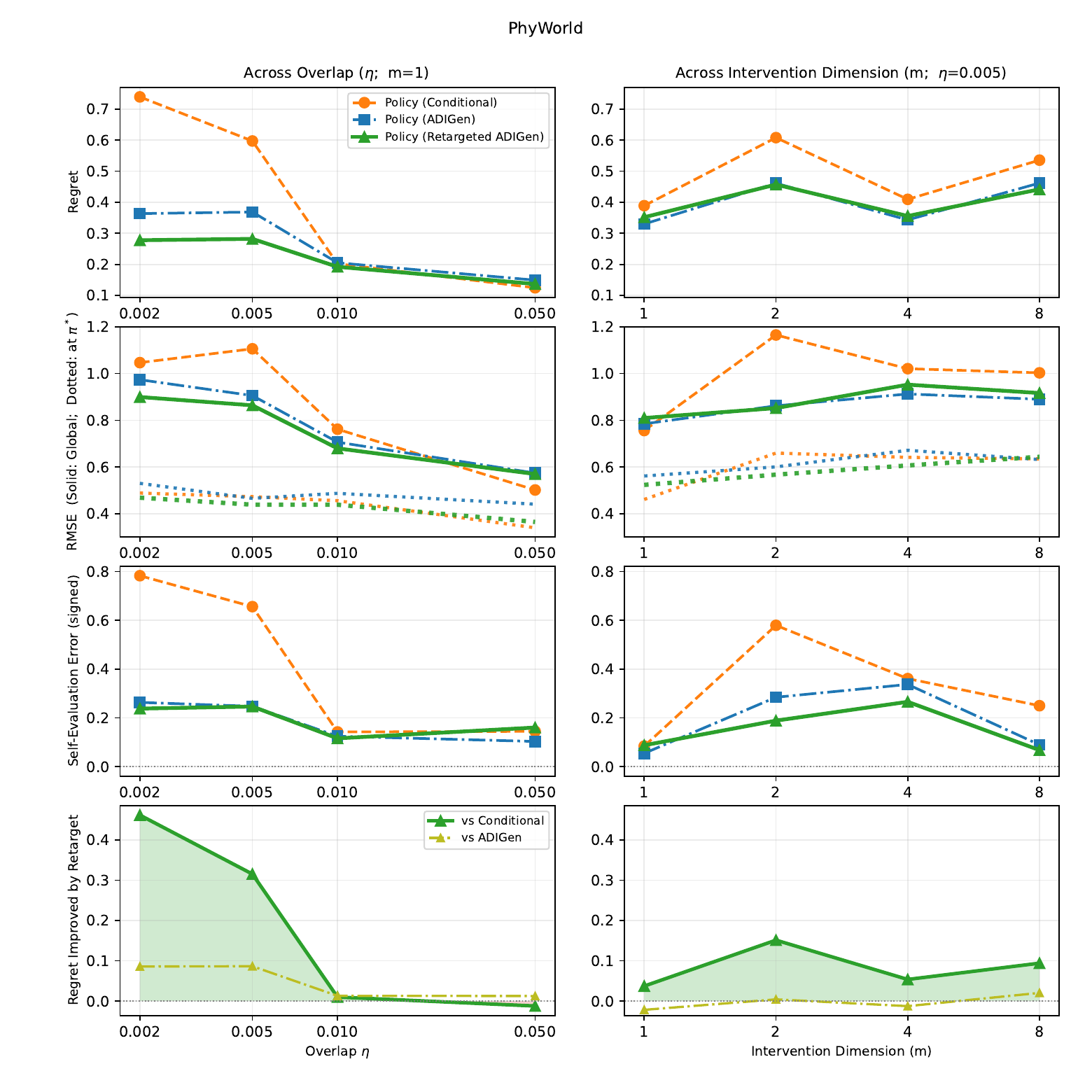}
    \caption{PhyWorld video world model at a matched training budget. Left: varying overlap $\eta$, the probability a strong-gust scene appears in the logs ($m=1$). Right: varying the number of balls $m$ ($\eta=0.005$). Rows, top to bottom: regret; model MSE globally (solid) and at $\pi^*$ (dotted); signed self-evaluation error; and regret reduction from retargeting relative to the conditional generator and ADIGen (P1--P3).}
    \label{fig:wm}
\end{figure}

\section{Conclusion}
We propose a novel policy learning algorithm based on generative modeling. In contrast to existing approaches, we ensure we generate faithful counterfactuals from an interventional law overcoming biased from logging data while maintaining stability and invariance. We show theoretically we may guard against a worst case logging mismatch and demonstrate the efficacy of our approach in extensive experiments across simulated, biological and video experiments, supporting a good generator does not always mean a good decision-maker.
\clearpage


\renewcommand{\refname}{}

\newpage 

\appendix
\thispagestyle{empty}
\onecolumn
\begin{center}
    \textbf{SUPPLEMENTARY MATERIAL TO\\ ``\titleString''}\\ \vspace{0.25cm}
    \normalsize Raphael C. Kim$^{1,2}$, Jingsen Zhu$^1$, Ramin Zabih$^1$, and Michele Santacatterina$^2$ \\
    {\small $^1$Cornell Tech, Cornell University, New York, NY}\\
    {\small $^2$Department of Biostatistics, Department of Population Health, New York University Grossman School of Medicine, New York, NY}
\end{center}

\section{Proofs}\label{sec:proofs}
\listsafesectioning
\subsection{Setup.}
Norms are over $L^2(\pi_b\times P_{\bX,E})$ if not specified. $\|\cdot\|_\nu$ is the norm under the target law $P_\nu$, in which $\bA\sim\nu(\cdot\mid\bX)$. Constants hold uniformly over $E\in\mathcal E$ and may depend on $m$.

\begin{enumerate}[label=(\textbf{A\arabic*}),leftmargin=*]
\item\label{ass:existence} \textbf{Oracle minimizer.} $\theta_0\in\arg\min_{\theta\in\Theta}R(\theta)$ exists and $\theta_0\in\Theta_\inv$.
\item\label{ass:bddLoss} \textbf{Bounded loss.} $\sup_{\theta\in\Theta}\|\Delta\ell(\theta)\|_\infty \leq C_2$.
\item\label{ass:curvature} \textbf{Curvature.} $\|\Delta\ell(\theta)\|^2_\nu\leq C_3(m) G(\theta)$ for all $\theta\in\Theta$.
\item\label{ass:entropy} \textbf{Entropy.} For fixed nuisances $\eta$, the class $\{L_\eta(\theta)-L_\eta(\theta_0): \theta\in\Theta_\inv\}$ has finite uniform entropy integral $J(\delta)$ for small $\delta$; the critical radius $\delta_{n,\inv}$ is the smallest $\delta$ with $J(\delta)\leq \sqrt n \delta^2$.
\item\label{ass:mixedLip} \textbf{Centered mixed-Lipschitz.} Let
$D_\psi(\theta)=\big[\zeta_{\psi_0}-\zeta_\psi\big](\theta)-\big[\zeta_{\psi_0}-\zeta_\psi\big](\theta_0)$ be the error of the candidate sampler $\psi$ relative to the true sampler $\psi_0$ centered by its value at $\theta_0$. Then, $\|D_\psi(\theta)\|^2_\nu\leq C_5(m) G(\theta) d_{\Psi,\inv}(\psi,\Psi_\inv)^2$ for all $\theta,\psi$.
\item\label{ass:nuisances} \textbf{Nuisances.} $\|1-\hat\alpha/\alpha_\nu\|_\nu=O_p(n^{-r_\alpha})$, $d_{\Psi,\inv}(\hat\psi,\Psi_\inv)=O_p(n^{-r_\psi})$, and $\|\hat\alpha\|_\infty \leq C_\alpha$, $\hat\alpha \leq C_w\alpha_\nu$.
\item\label{ass:consistentInv} \textbf{Learnable invariance penalty.} $\Inv(\theta_0)=0$, $\hat\Inv_n\ge0$, and $\sup_\theta|\hat\Inv_n(\theta)-\Inv(\theta)|\leq r_{n,\Inv}\to 0$ with high probability.
\item\label{ass:weights} \textbf{Weight size.} $\delta^2_{n,\inv}(\nu)\lesssim\big(1+\chi^2(\nu\|\pi_b)\big)\bar\delta^2_{n,\inv}$, where $\chi^2(\nu\|\pi_b)=\E_{\pi_b}[(\frac{\nu}{\pi_b}-1)^2]$ and $\bar\delta_{n,\inv}$ is the critical radius with no weighting.
\end{enumerate}
\begin{enumerate}[label=(\textbf{Int\arabic*}),leftmargin=*]
\item\label{ass:overlap} \textbf{Overlap over the joint intervention.} $\E_{\pi_b}\big[\big(\frac{\nu(\bA\mid\bX)}{\pi_b(\bA\mid\bX)}\big)^{2+\kappa}\big]<\infty$ for some $\kappa\ge0$.
\end{enumerate}
\begin{enumerate}[label=\textbf{(Inv\arabic*)},leftmargin=*]
\item\label{ass:invariance} $\pr_E\big(Y(\ba)\mid S_0(\bX),T_0(\ba)\big)=\pr_{E'}\big(Y(\ba)\mid S_0(\bX),T_0(\ba)\big)$ for all $E,E'\in\mathcal E$.
\end{enumerate}
Assumptions \labelcref{ass:existence}--\labelcref{ass:nuisances} are used for the generator, and are relatively standard regularity conditions for nonparametric guarantees \citet{Luedtke25}. \labelcref{ass:consistentInv} with realizability details that our problem is well-posed. \labelcref{ass:weights} is for policy guarantees, and is the non-standard assumption. At a high-level, this assumptions states the reweighting reduces the effective sample size, formally through the critical radius. Here, the doubly-robust loss weights each logged cluster by $\alpha_\nu$, and a weighted sample of size $n$ carries about as much information as an unweighted one of size $n/\E_{\pi_b}[\alpha_\nu^2]=n/(1+\chi^2(\nu\|\pi_b))$. The precise scaling depends on how the weights correlate with $\ell$, with a worst case of bounding by the largest weight.
\labelcref{ass:invariance} is a formal statement of causal invariance that our counterfactual laws that we learn are invariant across environments, and \labelcref{ass:overlap} states our densities (\URR) are well actually learnable.
 
\subsection{Empirical Process Bound}
\begin{lemma}\label{lemma:localizedEPBd}
Let $\FClass$ be a class of measurable functions with envelope $F$ satisfying $\|F\|_\infty \leq b<\infty$ and finite uniform entropy integral $J(\delta)$.  Let $\delta_n$ be the critical radius, which is the smallest $\delta>0$ with $J(\delta)\le\sqrt n\,\delta^2$. For $s>0$ let $s'=s+\log\big(2+\log_2(b\sqrt n)\big)$. Then, with probability at least $1-e^{-s}$, we have the following uniform bound
\[
\sup_{f \in \FClass}\big[ |(\E-\E_n)f| - \delta_n \|f\|_2 \big] \lesssim \delta_n^2+\frac{s'}{n}
\]
where the constant depends only on $b$.
\end{lemma}
 
\begin{proof}
The proof proceeds by (1) bounding the expected localized supremum in terms of the critical radius, and (2) show concentration to its expectation using peeling to conclude the result. For $r>0$, define the localized supremum
\[
    v_n(r)=\sup\big\{|(\E-\E_n)f|:f\in\FClass, \|f\|_2\leq r\big\}
\]

\paragraph{Step 1.}
By Theorem 2.1 of \citet{vdV11}, since $\|F\|_2\leq b$, we immediately observe that
\[
    \E v_n(r)\lesssim\frac{J(r)}{\sqrt n}\Big(1+\frac{J(r)}{\sqrt nr^2}\Big)
\]
We claim this expression can be bounded further by $\E v_n(r)\lesssim\delta_nr+\delta_n^2$ for all $r>0$. The entropy integral $J$ is concave, so $J(r)/r$ is non-increasing in $r$ (see Lemma S3 of \citet{Luedtke25}). 

\begin{itemize}
    \item For $r\ge\delta_n$, this and the definition of $\delta_n$ give
    \[
        \frac{J(r)}{r}\leq \frac{J(\delta_n)}{\delta_n}\leq \sqrt n \delta_n \quad \Longrightarrow 
        \quad\frac{J(r)}{\sqrt n} \leq \delta_nr, \qt \frac{J(r)}{\sqrt nr^2}\leq\frac{\delta_n}{r}\leq 1
    \]
    Hence, we have from the initial bound on $\E v_n(r)$, $\E v_n(r)\lesssim\delta_nr$. 
    \item For $r<\delta_n$, $v_n$ is non-decreasing in $r$, so $\E v_n(r)\le\E v_n(\delta_n)\lesssim\delta_n^2$.
\end{itemize}

\paragraph{Step 2.}
Now, we characterize the concentration of the localized supremum, and conclude the result. The functions in $\{f\in\FClass:\|f\|_2\le r\}$ are bounded by $b$ and have variance at most $r^2$, so Talagrand's concentration inequality for bounded empirical processes, Lemma 3.5.9 of \citet{vdv96} (Talagrand’s) gives the following. For any $t>0$, with probability at least $1-e^{-t}$,
\[
v_n(r)\lesssim\E v_n(r)+r\sqrt{\frac tn}+\frac{bt}{n}\lesssim\delta_nr+\delta_n^2+r\sqrt{\frac tn}+\frac tn
\]
based on the upper bound from Step 1. 
 
This bound holds for fixed $r$. To refine this bound, we take a peeling argument adapting to sizes of $f$ ($r$). Since $J(\delta)\geq \delta$ and $\|f\|_2\leq b$, the critical radius satisfies $\delta_n\geq 1/\sqrt n$.
Define the radii $r_k=2^k\delta_n$, $k=0,\dots,K$, with $K=\lceil\log_2(b/\delta_n)\rceil\leq 1+\log_2(b\sqrt n)$; in this way, the radii over $K$ covers the range $[0,b]$. 
Apply the bound above (stated for fixed $r$) at each radius $r_k$ with $t=s'=s+\log(K+1)$. 
Taking a union bound over all the $K+1$ events, we can conclude the bound holds for all radii with probability at least $1-(K+1)e^{-s'}=1-e^{-s}$.
On this high probability event, for $f\in\FClass$ set $k$ to be the smallest index with $\|f\|_2\leq r_k$. Then, $r_k\le2\|f\|_2+\delta_n$. Plugging into our result of interest, we can conclude with high probability that the following holds over any $f \in \FClass$,
\begin{align*}
    |(\E-\E_n)f| &\leq v_n(r_k) \\
    & \lesssim\Big(\delta_n+\sqrt{\tfrac{s'}{n}}\Big)\big(2\|f\|_2+\delta_n\big)+\delta_n^2+\frac{s'}{n} \\
    & \lesssim\Big(\delta_n+\sqrt{\tfrac{s'}{n}}\Big)\|f\|_2+\delta_n^2+\frac{s'}{n}
\end{align*}
where the last step follows noting $\delta_n\sqrt{s'/n}\le\frac12\delta_n^2+\frac12s'/n$.

\end{proof}

\subsection{Proof of \texorpdfstring{\cref{thm:adigen}}{Theorem 1}}\label{pf:adigen}
\begin{proof}

\paragraph{Decomposition.} Define the centered losses
\[
f_\theta=L_{\eta_0}(\theta)-L_{\eta_0}(\theta_0),\qquad \hat f_\theta=L_{\hat\eta}(\theta)-L_{\hat\eta}(\theta_0)
\]
with $\hat\eta$ cross-fitted, so that $\hat f_\theta$ is a fixed function given the other fold. Now, decompose the loss as follows, adding and subtracting moments:
\[
    G(\hat\theta)=\underbrace{\E\big[f_{\hat\theta}-\hat f_{\hat\theta}\big]}_{\text{I: DR Remainder}}+\underbrace{(\E-\E_n)\hat f_{\hat\theta}}_{\text{II: Empirical Process}}+\underbrace{\E_n\hat f_{\hat\theta}}_{\text{III: Invariance Penalty}}
\]

\paragraph{Term I.} Fix $\theta$ and nuisances $\eta=(\alpha,\psi,S,T)$. For any $\psi$, note the identity from invoking the Riesz Representer, $\E\bar\zeta_\psi(\theta;\bX)=\E[\alpha_\nu\zeta_\psi(\theta;\bX,\bA)]$.
Further, under no unmeasured confounding, $
\E[\ell(\theta;Y,\bA)\mid\bX,\bA,E]=\zeta_{\psi_0}(\theta;\bX,\bA)$, so $R(\theta)=\E_\nu[\ell(\theta;Y,\bA)]=\E[\alpha_\nu\zeta_{\psi_0}(\theta)]$. Hence, from \eqref{eq:drloss},
\[
    \E L_\eta(\theta)=\E\big[\alpha(\zeta_{\psi_0}-\zeta_\psi)(\theta)\big]+\E\big[\alpha_\nu\zeta_\psi(\theta)\big]
\]
Thus,
\begin{align*}
R(\theta)-\E L_\eta(\theta)&=\E\big[\alpha_\nu(\zeta_{\psi_0}-\zeta_\psi)(\theta)\big]-\E\big[\alpha(\zeta_{\psi_0}-\zeta_\psi)(\theta)\big]\\
&=\E\big[(\alpha_\nu-\alpha)(\zeta_{\psi_0}-\zeta_\psi)(\theta)\big]
\end{align*}
Apply this identity at $\hat\theta$ and at $\theta_0$ with $\eta=\hat\eta$ to simplify term $I$. Letting, $D=[\zeta_{\psi_0}-\zeta_{\hat\psi}](\hat\theta)-[\zeta_{\psi_0}-\zeta_{\hat\psi}](\theta_0)$, we find
\begin{align*}
\text{I}&=\big[R(\hat\theta)-\E L_{\hat\eta}(\hat\theta)\big]-\big[R(\theta_0)-\E L_{\hat\eta}(\theta_0)\big]\\
    &=\E\big[(\alpha_\nu-\hat\alpha)(\zeta_{\psi_0}-\zeta_{\hat\psi})(\hat\theta)\big]-\E\big[(\alpha_\nu-\hat\alpha)(\zeta_{\psi_0}-\zeta_{\hat\psi})(\theta_0)\big]\\
    &=\E\big[(\alpha_\nu-\hat\alpha) D \big]
\end{align*}
Factor out $\alpha_\nu$ and apply Riesz, 
\begin{align*}
\text{I}&=\E_\nu\Big[\Big(1-\frac{\hat\alpha}{\alpha_\nu}\Big)D\Big] \\
&\leq \varepsilon_\alpha \|D\|_\nu, \mbox{ by Cauchy--Schwarz} \\
&\leq C_5(m)^{1/2} \| 1-\frac{\hat{\alpha}}{\alpha_\nu} \|_{\nu} d_{\Psi,\inv}(\hat\psi,\Psi_\inv) G(\hat\theta)^{1/2}, \mbox{ by }\labelcref{ass:mixedLip}
\end{align*}

\paragraph{Term II.} Conditional on the other fold, the nuisances $\hat\eta$ are fixed. Define a function class,
\[
    \hat f_\theta=\hat\alpha \big[ \Delta\ell(\theta)-\Delta\zeta_{\hat\psi}(\theta) \big]+\Delta \bar\zeta_{\hat\psi}(\theta), \qquad \theta \in \Theta_\inv
\]

We proceed to bound the empirical process term based on the function class above in three steps: we demonstrate a bounded envelope, bound it using a maximal inequality, and refine the bound by the norm by curvature.

\emph{Step 1: Bounded Envelope.} By \labelcref{ass:bddLoss}, $|\Delta\ell|\leq C_2$, so $|\Delta\zeta_{\hat\psi}|,|\Delta\bar\zeta_{\hat\psi}|\leq C_2$.
Further, by \labelcref{ass:nuisances}, $\|\hat\alpha\|_\infty \leq C_\alpha$. 
Together, we can form a bounded envelope as the supremum is given by $|\hat f_\theta| \leq B$, where
\[
    B = C_\alpha(C_2+C_2)+C_2=C_2(2C_\alpha+1)
\]

\emph{Step 2: Maximal Inequality.} Clusters are i.i.d., so \cref{lemma:localizedEPBd} applies with the cluster as the unit. More specifically, for each fold $j$, conditional on the other fold, $\{\hat f_\theta:\theta\in\Theta_\inv\}$ is a fixed class with envelope $B=C_2(2C_\alpha+1)$ and the entropy condition on \labelcref{ass:entropy}. 
Hence, with $s'=s+\log(2+\log_2(B \sqrt{n}))$ as in \cref{lemma:localizedEPBd}, the lemma can be applied to $n/2$ clusters of fold $3-j$: With probability at least $1-e^{-s}$, uniformly over $\theta\in\Theta_\inv$,
\[
    |(\E-\E_n)\hat f_\theta|\lesssim\Big(\delta_{n,\inv}+\sqrt{\tfrac{s'}{n}}\Big)\|\hat f_\theta\|+\delta_{n,\inv}^2+\frac{s'}{n}
\]

\emph{Step 3: Refinement.} Finally, we bound $\|\hat f_\theta\|$ by $G(\theta)^{1/2}$, by bounding the terms $\Delta\zeta_{\hat\psi}$ and $\Delta\bar\zeta_{\hat\psi}$.

Note, $\zeta_{\psi_0}(\theta;\bX,\bA)=\E[\ell(\theta;Y,\bA)\mid\bX,\bA,E]$, so by Jensen's, $\|\Delta\zeta_{\psi_0}(\theta)\|_\nu\le\|\Delta\ell(\theta)\|_\nu$.
Note, 
\begin{align*}
D_{\hat\psi}(\theta) &=   \big[\zeta_{\psi_0}-\zeta_{\hat\psi}\big](\theta)-\big[\zeta_{\psi_0}-\zeta_{\hat\psi}\big](\theta_0) \\
&=\big[\zeta_{\psi_0}(\theta)-\zeta_{\psi_0}(\theta_0)\big]-\big[\zeta_{\hat\psi}(\theta)-\zeta_{\hat\psi}(\theta_0)\big] \\
&=\Delta\zeta_{\psi_0}(\theta)-\Delta\zeta_{\hat\psi}(\theta)
\end{align*}
Rearranging, $\Delta\zeta_{\hat\psi}(\theta)=\Delta\zeta_{\psi_0}(\theta)-D_{\hat\psi}(\theta)$. Then, 
\begin{align*}
\|\Delta\zeta_{\hat\psi}(\theta)\|^2_\nu &\leq 2\|\Delta\zeta_{\psi_0}(\theta)\|^2_\nu+2\|D_{\hat\psi}(\theta)\|^2_\nu\\
&\leq 2C_3(m)G(\theta)+2C_5(m) d_{\Psi,\inv}(\hat\psi,\Psi_\inv)^2 G(\theta), \mbox{by }\labelcref{ass:curvature}\mbox{ and }\labelcref{ass:mixedLip}
\end{align*}

Finally, for $\bar\zeta_{\hat\psi}(\theta;\bx)=\int\zeta_{\hat\psi}(\theta;\bx,\ba) \nu(d\ba\mid\bx)$, $\Delta\bar\zeta_{\hat\psi}(\theta;\bx)=\int\Delta\zeta_{\hat\psi}(\theta;\bx,\ba) \nu(d\ba\mid\bx)$ is an average of $\Delta\zeta_{\hat\psi}$ over the intervention. 
By Jensen's inequality at each $\bx$,
\[
    \big(\Delta\bar\zeta_{\hat\psi}(\theta;\bx)\big)^2 \leq \int\big(\Delta\zeta_{\hat\psi}(\theta;\bx,\ba)\big)^2 \nu(d\ba\mid\bx)
\]
Integrating both sides over $\bX\sim P_{\bX}$ gives
\[
    \|\Delta\bar\zeta_{\hat\psi}(\theta)\|^2_{L^2(P_{\bX})}\le\|\Delta\zeta_{\hat\psi}(\theta)\|^2_\nu\le2\big(C_3(m)+C_5(m) d_{\Psi,\inv}(\hat\psi,\Psi_\inv)^2\big)G(\theta)
\]

Now, we put this all together. $\hat f_\theta=\hat\alpha\big[\Delta\ell(\theta)-\Delta\zeta_{\hat\psi}(\theta)\big]+\Delta\bar\zeta_{\hat\psi}(\theta)$, and write $K(m)=2\big(C_3(m)+C_5(m)\,d_{\Psi,\inv}(\hat\psi,\Psi_\inv)^2\big)$ for the constant above, so that $\|\Delta\zeta_{\hat\psi}(\theta)\|^2_\nu$ and $\|\Delta\bar\zeta_{\hat\psi}(\theta)\|^2_{L^2(P_{\bX})}$ are both at most $K(m)G(\theta)$. 
Then
\begin{align*}
\|\hat f_\theta\|^2
&=\E\Big[\big(\hat\alpha[\Delta\ell-\Delta\zeta_{\hat\psi}]+\Delta\bar\zeta_{\hat\psi}\big)^2\Big] \\
&\leq 2 \E\big[\hat\alpha^2(\Delta\ell-\Delta\zeta_{\hat\psi})^2\big]+2 \E\big[(\Delta\bar\zeta_{\hat\psi})^2\big] \\
&\overset{(i)}{\leq} 2 C_\alpha C_w \E\big[\alpha_\nu(\Delta\ell-\Delta\zeta_{\hat\psi})^2\big]+2\|\Delta\bar\zeta_{\hat\psi}\|^2_{L^2(P_{\bX})} \\
&=2C_\alpha C_w\,\E_\nu\big[(\Delta\ell-\Delta\zeta_{\hat\psi})^2\big]+2\|\Delta\bar\zeta_{\hat\psi}\|^2_{L^2(P_{\bX})} \\
&\leq 4C_\alpha C_w\big(\|\Delta\ell(\theta)\|^2_\nu+\|\Delta\zeta_{\hat\psi}(\theta)\|^2_\nu\big)+2\|\Delta\bar\zeta_{\hat\psi}(\theta)\|^2_{L^2(P_{\bX})} \\
&\overset{(ii)}{\leq} 4C_\alpha C_w\big(C_3(m)+K(m)\big)G(\theta)+2K(m) G(\theta) 
\end{align*}
where $(i)$ follows from $\hat\alpha^2\le C_\alpha\hat\alpha\le C_\alpha C_w\alpha_\nu$, and $(ii)$ follows from \labelcref{ass:curvature}. 

Finally, substituting in, we find, uniformly over $\theta\in\Theta_\inv$,
\begin{align*}
|(\E-\E_n)\hat f_\theta| &\lesssim \Big(4C_\alpha C_w\big(3C_3(m)+2C_5(m)\,d_{\Psi,\inv}(\hat\psi,\Psi_\inv)^2\big)+4C_3(m)+4C_5(m) d_{\Psi,\inv}(\hat\psi,\Psi_\inv)^2\Big)^{1/2}\Big(\delta_{n,\inv}+\sqrt{\tfrac{s'}{n}}\Big)G(\theta)^{1/2}\\
&\qquad+\delta_{n,\inv}^2+\frac{s'}{n}
\end{align*}
The pilot distance enters the oracle term only as the multiplier $1+d_{\Psi,\inv}(\hat\psi,\Psi_\inv)^2$, which is bounded whenever the pilot is consistent or $d_\Psi$ is a bounded metric; its rate matters only through the doubly robust product $\varepsilon_\alpha^2d_{\Psi,\inv}^2$, where a slow pilot can be compensated by a fast Riesz estimate and conversely.

\paragraph{Term III.} Since $\hat\theta$ minimizes the penalized empirical risk and $\theta_0\in\Theta_\inv$, $\E_nL_{\hat\eta}(\hat\theta)+\lambda_\inv\hat\Inv_n(\hat\theta)\le\E_nL_{\hat\eta}(\theta_0)+\lambda_\inv\hat\Inv_n(\theta_0)$. Rearranging, and using $\hat\Inv_n\ge0$, $\Inv(\theta_0)=0$ and \labelcref{ass:consistentInv},
\[
    \E_n\hat f_{\hat\theta}\le\lambda_\inv\big(\hat\Inv_n(\theta_0)-\hat\Inv_n(\hat\theta)\big)\le\lambda_\inv\hat\Inv_n(\theta_0)\le\lambda_\inv\big(\Inv(\theta_0)+r_{n,\Inv}\big)=\lambda_\inv r_{n,\Inv}
\]

\paragraph{Conclusion.} Collecting Terms I--III, with
\[
    a=c_1\Big(\delta_{n,\inv}+\sqrt{\tfrac{s'}{n}}\Big)+c_2 \| 1- \frac{\hat{\alpha}}{\alpha_\nu} \|_\nu d_{\Psi,\inv}(\hat\psi,\Psi_\inv), \qquad
    b=c_3\Big(\delta_{n,\inv}^2+\frac{s'}{n}+\lambda_\inv r_{n,\Inv}\Big)
\]
we have $G(\hat\theta)\le aG(\hat\theta)^{1/2}+b$, where $G(\hat\theta)\ge0$ because $\theta_0$ minimizes $R$ over $\Theta\supseteq\Theta_\inv$. Young's inequality gives $aG^{1/2}\leq \tfrac12G+\tfrac12a^2$, so $\tfrac12G(\hat\theta)\leq \tfrac12a^2+b$. Then,
\[
    a^2\le3c_1^2 \delta_{n,\inv}^2+3c_1^2\frac{s'}{n}+3c_2^2 \| 1- \frac{\hat{\alpha}}{\alpha_\nu} \|_\nu^2d_{\Psi,\inv}(\hat\psi,\Psi_\inv)^2
\]
which gives the stated bound, with probability at least $1-2e^{-s}$ from Term II (and on \labelcref{ass:consistentInv}), and constants depending on $m$.

\end{proof}

\subsection{Proof of \texorpdfstring{\cref{thm:regret-tilt}}{}}

\begin{proof}\label{pf:regret-tilt}
Let the true objective be given by $J_\beta$, which is a function of the true $(\mu_0,\pi_b)$, and the learned objective be $\hat J$, given by
\[
    \hat J_\beta(\pi)=\E_{\bX}\Big[\int\pi\hat\mu d\ba-\beta\KL(\pi\|\hat\pi_b)\Big]
\]
which is a function of $(\hat \mu_0,\hat \pi_b)$.

Now, decompose the regret as follows
\begin{align*}
    \Reg(\hat\pi_\beta)&=\underbrace{\big[J_\beta(\pi^*_\beta)-\hat J_\beta(\pi^*_\beta)\big]}_{\text{(I)}}+\underbrace{\big[\hat J_\beta(\pi^*_\beta)-\hat J_\beta(\hat\pi_\beta)\big]}_{\text{(II)}}-\underbrace{\big[J_\beta(\hat\pi_\beta)-\hat J_\beta(\hat\pi_\beta)\big]}_{\text{(III)}} 
\end{align*}

\paragraph{Term II} 
$\hat\pi_\beta$ is the maximizer of this estimated objective, so $II \leq 0$. More formally, by Gibbs variational principle \cite{dv75}, 
\[
\hat{J}_\beta(\hat\pi_\beta)-\hat J_\beta(\pi)=\beta\E_{\bX}\KL(\pi\|\hat\pi_\beta)
\]

so, $II \leq 0$. 

\paragraph{Terms I and III}
We first note a simplification on the residual of the estimated objective function:
\begin{align*}
J_\beta(\pi)-\hat J_\beta(\pi) &=\E_{\bX}\Big[\int\pi\mu_0 d\ba-\beta\KL(\pi\|\pi_b)\Big]-\E_{\bX}\Big[\int\pi\hat\mu d\ba-\beta\KL(\pi\|\hat\pi_b)\Big] \\
&=\E_{\bX}\int\pi(\mu_0-\hat\mu) d\ba-\beta \E_{\bX}\big[\KL(\pi\|\pi_b)-\KL(\pi\|\hat\pi_b)\big]\\
&=\E_{\bX}\int\pi(\mu_0-\hat\mu) d\ba-\beta \E_{\bX}\int\pi\Big(\log\frac{\pi}{\pi_b}-\log\frac{\pi}{\hat\pi_b}\Big)d\ba\\
&=\E_{\bX}\int\pi(\mu_0-\hat\mu) d\ba-\beta\,\E_{\bX}\int\pi\log\frac{\hat\pi_b}{\pi_b}\,d\ba\\
&=\E_{\bX}\int\pi\Big(\mu_0-\hat\mu-\beta\log\frac{\hat\pi_b}{\pi_b}\Big)d\ba
\end{align*}

Then, we may simplify terms $I$ and $II$ as:
\begin{align*}
    I -III &= \E_\bX \int \pi^*_\beta \Big(\mu_0-\hat\mu-\beta\log\frac{\hat\pi_b}{\pi_b}\Big)d\ba - \E_\bX \int \hat{\pi}_\beta \Big(\mu_0-\hat\mu-\beta\log\frac{\hat\pi_b}{\pi_b}\Big)d\ba \\
    &\leq \E_\bX \int (\pi^*_\beta - \hat{\pi}_\beta)\Big(\mu_0-\hat\mu-\beta\log\frac{\hat\pi_b}{\pi_b}\Big)d\ba \\
    &= \E_\bX \int \hat{\pi}_\beta (\frac{\pi_\beta^*}{\hat{\pi}_\beta} - 1)  \Big(\mu_0-\hat\mu-\beta\log\frac{\hat\pi_b}{\pi_b}\Big) d\ba \\
    & \leq \Big\|\mu_0-\hat\mu-\beta\log\frac{\hat\pi_b}{\pi_b}\Big\|_{\hat\pi_\beta}\cdot\Big(\E_{\bX}\int\hat\pi_\beta\Big(\frac{\pi^*_\beta}{\hat\pi_\beta}-1\Big)^2d\ba\Big)^{1/2}, \quad \text{by Cauchy--Schwarz}\\
    & \leq \Big\|\mu_0-\hat\mu-\beta\log\frac{\hat\pi_b}{\pi_b}\Big\|_{\hat\pi_\beta}\cdot \big(C_M \E_{\bX}\int\hat\pi_\beta\Big(\frac{\pi^*_\beta}{\hat\pi_\beta}-1-\log\frac{\pi^*_\beta}{\hat\pi_\beta}\Big)d\ba\Big)^{1/2}, \quad \text{since }(r-1)^2 \leq C_M(1-r-\log(r)) \\
    & = \Big\|\mu_0-\hat\mu-\beta\log\frac{\hat\pi_b}{\pi_b}\Big\|_{\hat\pi_\beta}\cdot \Big(C_M \E_{\bX}\Big[\int\pi^*_\beta d\ba-\int\hat\pi_\beta d\ba-\int\hat\pi_\beta\log\frac{\pi^*_\beta}{\hat\pi_\beta} d\ba\Big]\Big)^{1/2} \\
    &= \Big\|\mu_0-\hat\mu-\beta\log\frac{\hat\pi_b}{\pi_b}\Big\|_{\hat\pi_\beta}\cdot \Big(C_M \E_{\bX}\Big[\int\hat\pi_\beta\log\frac{\hat\pi_\beta}{\pi^*_\beta} d\ba\Big]\Big)^{1/2} \\
    &= \Big\|\mu_0-\hat\mu-\beta\log\frac{\hat\pi_b}{\pi_b}\Big\|_{\hat\pi_\beta}\cdot \Big(C_M \E_{\bX}\KL(\hat\pi_\beta\|\pi^*_\beta))^{1/2} \\
    &= \Big\|\mu_0-\hat\mu-\beta\log\frac{\hat\pi_b}{\pi_b}\Big\|_{\hat\pi_\beta} \cdot\Big(\frac{C_M}{\beta}\Reg(\hat\pi_\beta)\Big)^{1/2}, \quad \mbox{by Gibb's}
\end{align*}

Divide by the $\Reg$ term (if $\Reg$ is 0, we are done, so $\Reg \neq 0$), to conclude
\begin{align*}
    \Reg(\hat{\pi}_\beta) &\leq \frac{C_M}{\beta}\Big\|\mu_0-\hat\mu-\beta\log\frac{\hat\pi_b}{\pi_b}\Big\|_{\hat\pi_\beta}^2
\end{align*}

If $\Reg(\hat\pi_\beta)=0$ there is nothing to prove. Otherwise, dividing both sides by $\Reg(\hat\pi_\beta)^{1/2}$,
\[
\Reg(\hat\pi_\beta) \leq \Big(\frac{C_M}{\beta}\Big) \Big\|\mu_0-\hat\mu-\beta\log\frac{\hat\pi_b}{\pi_b}\Big\|_{\hat\pi_\beta}^2 
\]

Now, conclude
\begin{align*}
    \Reg(\hat\pi_\beta) &\leq \Big(\frac{C_M}{\beta}\Big) \Big\|\mu_0-\hat\mu-\beta\log\frac{\hat\pi_b}{\pi_b}\Big\|_{\hat\pi_\beta}^2 \\
    & \leq \frac{C_M}{\beta} (2 \| \mu_0 - \hat{\mu} \|_{\hat{\pi}_\beta}^2 + 2 \beta^2 \| \log (\frac{\hat{\pi}_b}{\pi_b}) \|_{\hat{\pi}_\beta}^2) \\
    & = \frac{2 C_M}{\beta} \| \mu_0 - \hat{\mu} \|_{\hat{\pi}_\beta}^2 + 2 \beta \| \log (\frac{\hat{\pi}_b}{\pi_b}) \|_{\hat{\pi}_\beta}^2
\end{align*}

Finally, by $(u+v)^2\le2u^2+2v^2$ applied pointwise with $u=\mu_0-\hat\mu$ and $v=-\beta\log(\hat\pi_b/\pi_b)$,
\[
\Big\|\mu_0-\hat\mu-\beta\log\frac{\hat\pi_b}{\pi_b}\Big\|^2_{\hat\pi_\beta}\le2\|\mu_0-\hat\mu\|^2_{\hat\pi_\beta}+2\beta^2\Big\|\log\frac{\hat\pi_b}{\pi_b}\Big\|^2_{\hat\pi_\beta},
\]
so that
\[
\Reg(\hat\pi_\beta)\le\frac{2C_M}{\beta}\|\hat\mu-\mu_0\|^2_{\hat\pi_\beta}+2C_M\beta\Big\|\log\frac{\hat\pi_b}{\pi_b}\Big\|^2_{\hat\pi_\beta}. \qedhere
\]

\end{proof}

\subsection{Proof of \texorpdfstring{\cref{thm:regret-retarg}}{}}

\begin{proof}\label{pf:regret-retarg}
By \cref{thm:regret-tilt},
\begin{align*}
\Reg_{\beta_j}(\hat\pi_j)&\le\frac{2C_M}{\beta_j}\|\hat\mu_j-\mu_0\|^2_{\hat\pi_j}+2C_M\beta_j\Big\|\log\frac{\hat\pi_b^{(3-j)}}{\pi_b}\Big\|^2_{\hat\pi_j}\\
&\le\frac{2C_M}{\beta_j}\kappa_j\|\hat\mu_j-\mu_0\|^2_{\pi_{j,\lambda}}+2C_M\beta_j\Big\|\log\frac{\hat\pi_b^{(3-j)}}{\pi_b}\Big\|^2_{\hat\pi_j}\\
&\overset{(i)}{\le}\frac{4C_MB^2}{\beta_j}\kappa_j\Big(C_D\mathcal R_{n_j}(\pi_{j,\lambda})^b+\epsilon_{\pi_{j,\lambda}}\Big)+2C_M\beta_j\Big\|\log\frac{\hat\pi_b^{(3-j)}}{\pi_b}\Big\|^2_{\hat{\pi}_j}
\end{align*}
where $(i)$ follows from bounding the outcome error in terms of the generator error where we invoke independence of the splits with union bound. 

The final step is to bound $\mathcal R_{n_j}(\pi_{j,\lambda}) \lesssim \mathcal R_{n_j}$. We demonstrate this by bounding the $\chi^2$ divergence between the mixture policy $\pi_{j,\lambda}$ and the behavioral policy $\pi_b$, and utilizing weight regularity \labelcref{ass:weights}.

\begin{align*}
\chi^2(\pi_{j,\lambda}\|\pi_b)&=\E_{\pi_b}\Big[\Big(\lambda\Big(\frac{\tilde\pi_{3-j}}{\pi_b}-1\Big)+(1-\lambda)\Big(\frac{\hat\pi_b^{(3-j)}}{\pi_b}-1\Big)\Big)^2\Big] \\
&\le2\lambda^2\chi^2(\tilde\pi_{3-j}\|\pi_b)+2(1-\lambda)^2\chi^2(\hat\pi_b^{(3-j)}\|\pi_b)&&(u+v)^2\le2u^2+2v^2\\
&\le2\lambda^2\chi^2(\tilde\pi_{3-j}\|\pi_b)+2\chi^2(\hat\pi_b^{(3-j)}\|\pi_b)\\
&=\chi^2_j
\end{align*}
By the weight-size assumption \labelcref{ass:weights}, the critical radius at the retargeting law is
\[
\delta^2_{n_j,\inv}(\pi_{j,\lambda})\le C_9\big(1+\chi^2(\pi_{j,\lambda}\|\pi_b)\big)\bar\delta^2_{n_j,\inv}\le C_9(1+\chi^2_j)\bar\delta^2_{n_j,\inv}
\]

Hence $\mathcal R_{n_j}(\pi_{j,\lambda})\le\mathcal R_{n_j}$ yielding the result.

\paragraph{\GTF bound.} The bound on $\kappa_j=\|\hat\pi_j/\pi_{j,\lambda}\|_\infty$ follows by lower bounding the policy by each of the mixture components
\[
\pi_{j,\lambda}\ge(1-\lambda)\hat\pi_b^{(3-j)}\qquad\text{and}\qquad\pi_{j,\lambda}\ge\lambda\tilde\pi_{3-j}
\]

For the first expression, since $\hat\pi_j/\hat\pi_b^{(3-j)}=e^{\hat\mu_j/\beta_j}/\hat Z$ and $|\hat\mu_j|\le B$,
\[
\kappa_j\le\frac{1}{1-\lambda}\Big\|\frac{e^{\hat\mu_j/\beta_j}}{\hat Z}\Big\|_\infty\le\frac{e^{2B/\beta_j}}{1-\lambda}
\]

From the second, if $\tilde\pi_{3-j}\propto\hat\pi_b^{(3-j)}e^{\tilde\mu_{3-j}/\beta_j}$, 
\[
\kappa_j\le\frac1\lambda\Big\|e^{(\hat\mu_j-\tilde\mu_{3-j})/\beta_j}\frac{\tilde Z}{\hat Z}\Big\|_\infty\le\frac{e^{2\|\hat\mu_j-\tilde\mu_{3-j}\|_\infty/\beta_j}}{\lambda}
\]

\end{proof}

\subsection{Proof of \texorpdfstring{\cref{thm:retarget-cmp}}{}}

\begin{proof}\label{pf:retarget-cmp}
    \paragraph{Preliminaries}
    We establish a couple of identities that we will use in the proof.
    
    First, a bound on $\chi^2$:
    \begin{align*}
        \chi^2&=\E_{\pi_b}\Big[\Big(\frac{\hat\pi}{\pi_b}\Big)^2\Big]-1\\
        &=\E_{\hat\pi}\Big[\frac{\hat\pi}{\pi_b}\Big]-1\\
        &\leq \rho-1 
    \end{align*}

    Second, the retargeting law has transfer factor at most $1/\lambda$ and weight cost $1+\lambda^2\chi^2$:
    \begin{align*}
    \pi_\lambda \ge\lambda\hat\pi\quad &\Longrightarrow \rho_{\pi_\lambda}=\Big\|\frac{\hat\pi}{\pi_\lambda}\Big\|_\infty\le\frac1\lambda\\
    \frac{\pi_\lambda}{\pi_b}-1=\lambda\frac{\hat\pi}{\pi_b}+(1-\lambda)&-1=\lambda\Big(\frac{\hat\pi}{\pi_b}-1\Big)\\
    &\Longrightarrow\quad1+\chi^2(\pi_\lambda\|\pi_b)=1+\lambda^2\chi^2 
    \end{align*}
    
    Combining the two, for every $\lambda\in(0,1]$,
    \[
    \B(\pi_\lambda)\le\frac1\beta\cdot\frac1\lambda\Big[a+(1+\lambda^2\chi^2)d\Big]=\frac1\beta\Big[\frac{a+d}{\lambda}+\lambda\chi^2d\Big] \tag{$\star$}
    \]

    \paragraph{Part A.}

    We compute the numerator and denominator. 

    For the numerator, plug-in and observe
    \begin{align*}
    \B(\pi_b)&=\frac\rho\beta(a+d)
    \end{align*}

    For the denominator, we need to optimize.
    Let $f(\lambda)=(a+d)/\lambda+\lambda\chi^2d$ and note $\B(\pi_\lambda)\le f(\lambda)/\beta$. 
    The deriative is given by
    \[
    f'(\lambda)=-\frac{a+d}{\lambda^2}+\chi^2d,
    \]
    which is optimized at $\lambda^*=\sqrt{(a+d)/(\chi^2d)}$. Plugging in, we find 
    \[
    f(\lambda^*)=(a+d)\sqrt{\frac{\chi^2d}{a+d}}+\sqrt{\frac{a+d}{\chi^2d}}\chi^2d=2\sqrt{(a+d)\chi^2d}
    \]

    Computing the ratio, we have $\chi^2\leq\rho$,
    \begin{align*}
        \frac{\B(\pi_b)}{\B(\pi_{\lambda^*})}&\geq \frac{\rho(a+d)}{2\sqrt{(a+d)\chi^2d}}\\
        &=\frac\rho2\sqrt{\frac{a+d}{\chi^2d}}\\
        &\geq \frac\rho2\sqrt{\frac{a+d}{\rho d}}\\
        &=\frac{\sqrt\rho}2\sqrt{1+\frac ad}
    \end{align*}

    \paragraph{Part B}
    Let $\nu$ be any interventional law:
    \begin{align*}
        t &= \rho_\nu=\Big\|\frac{\hat\pi}{\nu}\Big\|_\infty \ge1 \implies \nu\ge\hat\pi/t\ge0
    \end{align*}

     Observe the divergence between this law and the logged policy is given by
        \begin{align*}
            1+\chi^2(\nu\|\pi_b)&=\E_{\bX}\int\frac{\nu^2}{\pi_b} d\ba \\
            &\ge\E_{\bX}\int\frac{(\hat\pi/t)^2}{\pi_b} d\ba \\
            &=\frac{1+\chi^2}{t^2}
        \end{align*}

    Plugging in this divergence into $\B$, we find
    \begin{align*}
    \B(\nu)&=\frac1\beta\,t\Big[a+\big(1+\chi^2(\nu\|\pi_b)\big)d\Big] \\
    &\ge\frac1\beta\Big[ta+t\max\Big\{1,\frac{1+\chi^2}{t^2}\Big\}d\Big] \\
    &=\frac1\beta\Big[ta+\max\Big\{t,\frac{1+\chi^2}{t}\Big\}d\Big]
    \end{align*}

    Compare this to takin the mixture distribution, and we find for $\lambda=1/t$,
    \begin{align*}
    \B(\pi_{1/t})&\le\frac1\beta\Big[\frac{a+d}{1/t}+\frac1t\chi^2d\Big] \\
    &=\frac1\beta\Big[ta+\Big(t+\frac{\chi^2}{t}\Big)d\Big]\\
    &\le\frac2\beta\Big[ta+\max\Big\{t,\frac{1+\chi^2}{t}\Big\}d\Big] \\
    &\le2 \B(\nu)
    \end{align*}
    demonstrating our result.
\end{proof}

\subsection{Proof of \texorpdfstring{\cref{thm:inv}}{}}

\begin{proof}\label{pf:inv}
    \emph{Decomposition.} Apply \cref{thm:regret-tilt} with reference $\pi_b$, truth $m^{e'}$ and estimate $\hat\mu_j$, to find our regret in environment $e'$ is given by the following. Letting $\|\cdot\|_{e'}=\|\cdot\|_{P^{e'}\otimes\hat\pi_j}$,
    {\tight \begin{align*}
    \Reg^{e'}_{\beta_j}(\hat\pi_j)
    &\leq\frac{2C_{M'}}{\beta_j}\|m^{e'}-\hat\mu_j\|^2_{e'}+2C_{M'}\beta_j\Big\|\log\frac{\hat\pi_b}{\pi_b}\Big\|^2_{e'}\\
    &\leq\frac{4C_{M'}}{\beta_j}\|m^{e'}-\bar m_{\hat\theta}\|^2_{e'}+\frac{4C_{M'}}{\beta_j}\|\bar m_{\hat\theta}-\hat\mu_j\|^2_{e'}\\
    &\quad+2C_{M'}\beta_j\Big\|\log\frac{\hat\pi_b}{\pi_b}\Big\|^2_{e'}
    \end{align*}}
    We bound the three terms one by one. 

    \emph{Generator term.} By change of measure and \cref{cor:mu-error},
    {\tight \begin{equation*}
        \frac{4C_{M'}}{\beta_j}\|\bar m_{\hat\theta}-\hat\mu_j\|^2_{P^{e'}\otimes\hat\pi_j}\leq\frac{4C_{M'}}{\beta_j}\rho_{\mathrm{gen}}\cdot2B^2\mathcal G_{n_j}(\pi_{j,\lambda})=\frac{8C_{M'}B^2}{\beta_j}\rho_{\mathrm{gen}}\mathcal G_{n_j}(\pi_{j,\lambda})
    \end{equation*}}

    \emph{Propensity term.} The integrand depends on the context through $S$, so
    {\tight \begin{equation*}
    2C_{M'}\beta_j\Big\|\log\frac{\hat\pi_b}{\pi_b}\Big\|^2_{P^{e'}\otimes\hat\pi_j}\leq2C_{M'}\beta_j\rho_{\mathrm{ctx}}\mathcal P_{n_j}(\hat\pi_j).
    \end{equation*}}
    The Generator and Propensity term sum to $\mathcal T_j$.

    \emph{Shift term, general target.} Since $\bar m_{\hat\theta}=\sum_eP(e\mid S)m^e_{\hat\theta}$ with weights summing to one,
    {\tight \begin{equation*}
    |m^{e'}-\bar m_{\hat\theta}|=\Big|\sum_eP(e\mid S)(m^{e'}-m^e_{\hat\theta})\Big|\leq\sup_{e,e''\in\mathcal E}\|m^e_{\hat\theta}-m^{e''}_{\hat\theta}\|_\infty
    \end{equation*}}

    Combined, this gives the first result.

    For the second, consider under invariance the mixture target. Let $q_e(S)=q_e\frac{dP^e_S}{dP^{e'}_S}(S)$, so $\sum_eq_e(S)=1$ and $m^{e'}=\sum_eq_e(S)m^e_{\hat\theta}$. Then
    {\tight \begin{align*}
    \|m^{e'}-\bar m_{\hat\theta}\|^2_{P^{e'}\otimes\hat\pi_j}
    &\leq\rho(\hat\pi_j\|\pi_{j,\lambda})\|m^{e'}-\bar m_{\hat\theta}\|^2_{P^{e'}\otimes\pi_{j,\lambda}}\\
    &=\rho(\hat\pi_j\|\pi_{j,\lambda})\Big\|\sum_eq_e(S)(m^e_{\hat\theta}-\bar m_{\hat\theta})\Big\|^2_{P^{e'}\otimes\pi_{j,\lambda}}\\
    &\leq\rho(\hat\pi_j\|\pi_{j,\lambda})\sum_e\int q_e(S)\|m^e_{\hat\theta}-\bar m_{\hat\theta}\|^2_{\pi_{j,\lambda}}dP^{e'}_S\\
    &=\rho(\hat\pi_j\|\pi_{j,\lambda})\sum_eq_e\|m^e_{\hat\theta}-\bar m_{\hat\theta}\|^2_e\\
    &\leq\rho(\hat\pi_j\|\pi_{j,\lambda})\rho(q\|p)\sum_ep_e\|m^e_{\hat\theta}-\bar m_{\hat\theta}\|^2_e.
    \end{align*}}
    where the third line uses Jensen's inequality in $e$ at fixed $S$, the fourth line uses $q_e(S)dP^{e'}_S=q_edP^e_S$, and the final line uses $q_e\leq\rho(q\|p)p_e$. Multiplying by $4C_{M'}/\beta_j$ gives the second bound.
\end{proof}
\section{Experiment Details}\label{app:exp}

\subsection{Setup}\label{app:exp-common}
\paragraph{Setup.} We observe $n$ i.i.d. clusters $(E,X,A,Y)$ logged under a policy $\pi_b(a\mid x)$, where $E$ is the environment, $X$ the covariates, $A=(A_1,\dots,A_m)$ the joint intervention on $m$ units, and $Y$ the outcome. A utility $g(Y)$ defines $\mu_0(x,a)=\E[g(Y(a))\mid X=x]$ and $V(\pi)=\E_\pi\mu_0$.

\paragraph{Algorithms.} Our retargeting procedure has three steps. (1) Fit the generator on the logged data. (2) Learn the KL-regularized policy $\hat\pi_\beta(a\mid x)\propto\hat\pi_b(a\mid x)\exp(\hat\mu(x,a)/\beta)$, with $\hat\mu$ estimated from generated outcomes. (3) Retarget the generator at $\pi_\lambda=\lambda\hat\pi+(1-\lambda)\pi_b$, by fine-tuning with the doubly robust loss, and relearn the policy. The policy step differs only in the action space: in the synthetic setting $\hat\pi_\beta$ is projected onto a linear-Gaussian class by weighted maximum likelihood over $128$ logger candidates per context; in RxRx19a it is computed exactly over the $256$ (compound, dose) interventions; in PhyWorld it is computed over a bank of $32$ candidate kicks.
The conditional generator runs steps 1--2 without reweighting. \ADG runs them with the doubly robust loss at a fixed law $\nu_0$. Retargeted \ADG{} runs all three steps, starting from \ADG{}, and $+$Inv adds the invariance penalty at weight $\lambda_\inv$. Within an experiment, all arms share the architecture, policy class and training budget, and $\beta$ is selected per arm.

\paragraph{Metrics.} The metrics are given by
\begin{align*}
\text{Regret}&=V(\pi^\star)-V(\hat\pi) &
\text{Error at }\pi^\star&=\|\hat\mu-\mu_0\|_{L^2(\pi^\star)}\\
\text{Global error}&=\|\hat\mu-\mu_0\|_{L^2(\mathrm{Unif}(\mathcal A))} &
\text{Self-evaluation}&=\hat V(\hat\pi)-V(\hat\pi)
\end{align*}
with $\hat\mu$ and $\hat V$ computed from generated samples and $V$ from each setting's oracle. Self-evaluation is signed: positive values mean the generator overrates its own policy.

\paragraph{Cross-Validation.} $\beta$ and $\lambda$ are selected by cross-validation on the the doubly robust value estimate
\begin{equation}\label{eq:cv-value}
\hat V_{\mathrm{val}}(\pi)=\frac{1}{|\mathcal D_{\mathrm{val}}|}\sum_{i\in\mathcal D_{\mathrm{val}}}
\Big[\E_{a\sim\pi(\cdot\mid X_i)}\bar\mu(X_i,a)+\frac{\pi(A_i\mid X_i)}{\hat\pi_b(A_i\mid X_i)}\big(g(Y_i)-\bar\mu(X_i,A_i)\big)\Big],
\end{equation}
where $\bar\mu$ is an evaluation outcome model fit on the training split. $\beta$ is selected per arm, and $(\beta,\lambda)$ jointly for retargeted arms. The grids range from $\beta\in\{0.05,0.1,0.2,0.5,1\}$ and $\lambda\in\{0.25,0.5,0.75,1\}$.

\paragraph{Doubly robust loss.} Let $\ell_\theta(y;x,a)$ be the per-sample flow-matching loss. Given a target law $\nu$, the density ratio $\alpha=\nu/\hat\pi_b$ and a pilot generator $\hat\psi$, we minimize
\begin{equation}\label{eq:dr-loss}
\hat L_{\mathrm{DR}}(\theta)=\frac1n\sum_{i=1}^n\Big[\E_{a'\sim\nu,\ \tilde y'\sim\hat\psi(\cdot\mid X_i,a')}\ell_\theta(\tilde y';X_i,a')
+\hat\alpha_i\big\{\ell_\theta(Y_i;X_i,A_i)-\E_{\tilde y\sim\hat\psi(\cdot\mid X_i,A_i)}\ell_\theta(\tilde y;X_i,A_i)\big\}\Big].
\end{equation}
The first term is estimated by samples from the target actions, and the second corrects the pilot sampler's error on the logged rows. For retargeting at $\pi_\lambda$, according to \cref{algo:retarget}, the ratio has the closed form
\begin{equation}\label{eq:retarget-alpha}
\alpha_\lambda(x,a)=\lambda\,\frac{e^{\hat\mu(x,a)/\beta}}{\hat Z_\beta(x)}+(1-\lambda),\qquad
\hat Z_\beta(x)=\frac1K\sum_{k=1}^K e^{\hat\mu(x,a_k)/\beta},\quad a_k\sim\hat\pi_b(\cdot\mid x)
\end{equation}
so only $Z$ constant must be estimated. Pilots are cross-fit across folds, and weights are capped at the 99th percentile.

\paragraph{Invariance penalty.} Let $\Delta_\theta(x,a;e)$ be the model's contrast between action $a$ and a reference action in environment $e$. The penalty we utilize based on our discussion after \cref{thm:inv} is
\begin{equation}\label{eq:inv}
R_\inv(\theta)=\frac{1}{|\mathcal P|}\sum_{(e,e')\in\mathcal P}\E\big[\big(\Delta_\theta(X,A;e)-\Delta_\theta(X,A;e')\big)^2\big]
\end{equation}
added to the training loss with weight $\lambda_\inv$. It vanishes exactly when the causal contrasts are invariant across environments.
Environments are three outcome shifts (synthetic)
and experimental batches (RxRx19a). Contrasts are taken on the mean outcome in the synthetic setting and on the velocity field, at shared noise and time, in the image settings.

\paragraph{Seeds and uncertainty.} Synthetic results use 20 seeds. Each seed redraws the data and all Monte Carlo draws, and all arms share it. Intervals are $\pm1.96$ standard errors of per-seed paired differences.

\subsection{Synthetic Data}\label{app:sim}
  \paragraph{Data.} Each of $n=10{,}000$ clusters (split $60/20/20$ into train, validation and test) has $m\in\{1,2,4,8\}$ units with covariates $x_j\sim U(-1,1)$ and two-dimensional actions $a_j$ (intervention dimension $d_A=2m$). The outcome is a single bump around the unit-wise optimum $a^\star_j=(-x_j,\tfrac12x_j)$ of
  \[
  \mu_0(x,a)=5\Big[1.5\exp\Big(-\frac{1}{m}\sum_{j=1}^m\frac{\|a_j-a^\star_j\|^2}{0.5}\Big)-\frac{0.02}{m}\sum_{j=1}^m\|a_j\|^2\Big]
  \]
  where $Y=\mu_0(X,A)+\varepsilon, \varepsilon\sim\mathcal N(0,1.5^2)$. The test set is a fresh draw from this set.

\paragraph{Biased logging.} The logger draws $A_j\sim \mathcal N\big((1-s)a^\star_j+s(x_j,x_j),I\big)$. At $s=0$ it is centered on the optimum, and larger $s$ moves it toward the incorrect mean $(x_j,x_j)$. We report overlap as $\eta$, the $99^{\text{th}}$ percentile of $\pi^\star/\pi_b$ under the true logger, where $\pi^\star$ places a Gaussian of width $0.3$ at $a^\star$.

\paragraph{Why this is hard.} The setting is unforgiving in two ways. First, the reward is multiplicative across units so one unit far from its optimum lowers the reward for all. Second, overlap degrades along two axes. Increasing $s\in[0,1.5]$ at $m=4$ pushes the logger away from the optimum, and increasing $m$ at $s=1$ harms the density ratio, which is a product over units. Hence $\log\eta$ grows linearly in $d_A$ and quadratically in $s$.

\paragraph{Models.} The incorrect outcome model is given by
\begin{align*}
\hat\mu_\theta(x,a) &= \theta_0 + \frac{1}{m}\sum_{j=1}^{m} \theta^\top \phi(x_j, a_j) \\
\phi(x_j, a_j) &= \big(a_{j1}, a_{j2}, a_{j1}^2, a_{j2}^2, a_{j1}a_{j2}, x_j a_{j1}, x_j a_{j2}, x_j a_{j1}^2, x_j a_{j2}^2, x_j, x_j^2\big)
\end{align*}
a quadratic in $a_j$ with coefficients linear in $x_j$, additive over units. In the correctly specified model, we append the true feature $\exp\{-\frac1m\sum_j\|a_j-a^\star_j\|^2/0.5\}$ (\cref{fig:toy_main_correct}).

\begin{figure}[ht!]
    \centering
    \includegraphics[width=\linewidth]{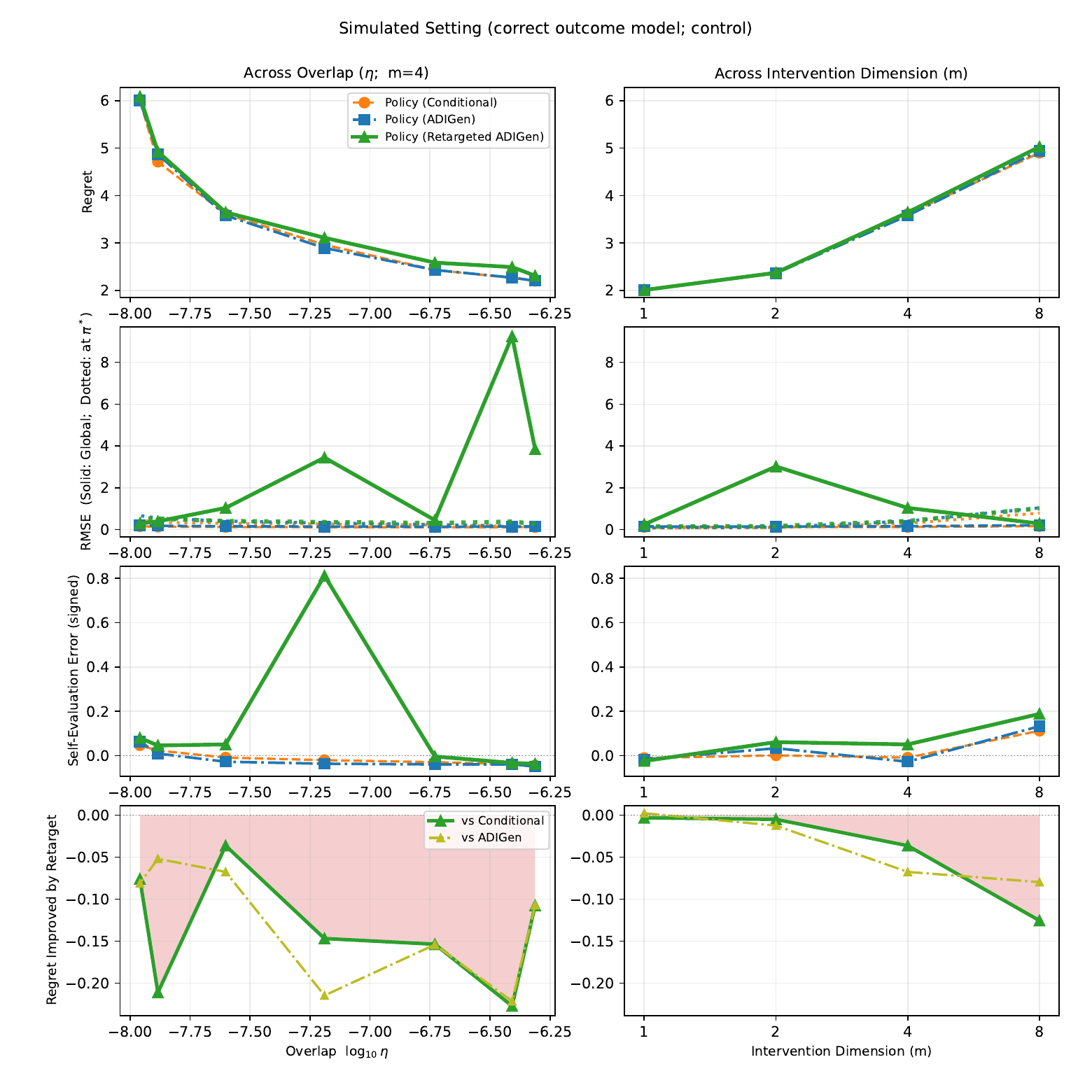}
    \caption{Synthetic data, correctly specified outcome model (control). Layout as in Figure~\ref{fig:toy_main_misspec}. With a well-specified model, all methods attain similar regret and retargeting yields no improvement (paired differences within noise).}
    \label{fig:toy_main_correct}
\end{figure}

\paragraph{Environment-aware training.} Three environments are encoded as follows: the logger differs $A_j\sim\mathcal N(\gamma_e(x_j,x_j),I)$ with $\gamma_e\in\{-1,0,1\}$, and the environment experiences a constant, environment-dependent shift $5c_e\bar x$ with $c_e\in\{-1,0,1\}$. Under this premise, the causal contrast is invariant across environments.

\begin{figure}[ht!]
    \centering
    \includegraphics[width=\linewidth]{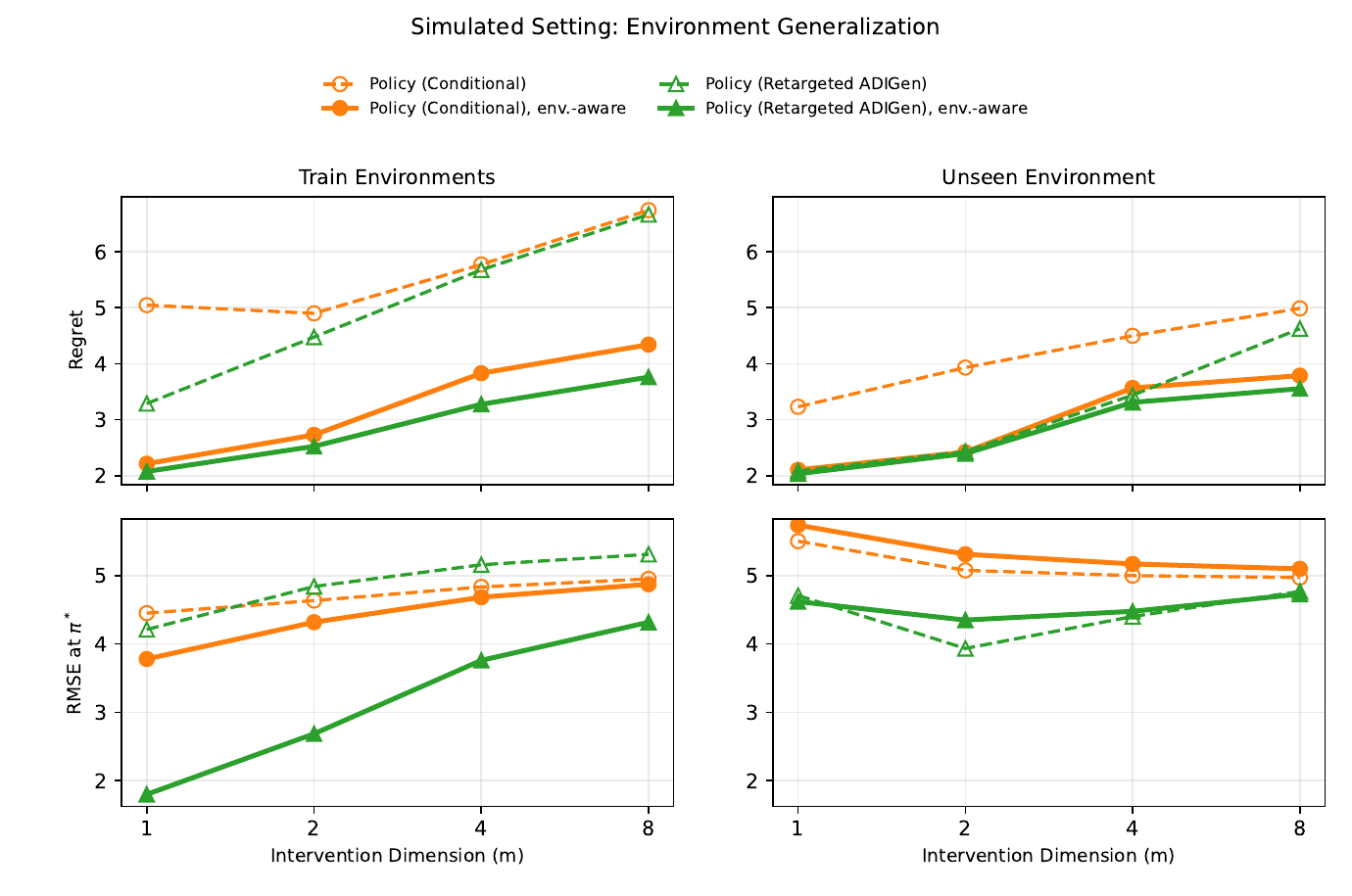}
    \caption{Synthetic data, generalization across environments. Left: the three training environments; right: an unseen environment, excluded from training and selection (generators are trained on two of the three). Dashed, open markers: environment-blind generators; solid, filled: environment-aware generators with the invariance penalty, which reduces the measured invariance gap from $0.5$--$0.7$ (environment-aware, unpenalized) to $0.006$--$0.02$ (not shown). Top: regret against the number of units $m$; environment-aware training lowers regret in both settings, and environment-aware retargeted ADIGen attains the lowest regret. Bottom: RMSE at $\pi^\star$; retargeting is more accurate where the policy acts on the training environments, while on the unseen environment environment-aware training improves decisions more than accuracy (P4).}
    \label{fig:toy_inv}
\end{figure}

\subsection{RxRx19a}\label{app:exp-rx}
\paragraph{Data.} RxRx19a screens two cell lines, human renal cortical epithelial cells (HRCE) and African green monkey kidney cells (VERO), each in two experiment batches. Each well is imaged at four sites with five fluorescence channels. We downsample images to $128\times128\times5$, and restrict to wells under active SARS-CoV-2 infection. We further consider a panel of $32$ compounds screened in both lines with untreated infected (vehicle) wells. Each compound is tested at $8$ doses, giving $256$ (compound, dose) interventions. This 32 compound reduction was done to leverage the natural environment and confounding structure in the experiment across the cell types and wells, while still maintaining a challenging high-dimensional setting.

Data are split, grouped by well and stratified by (compound, dose and cell line). A reserve of three wells per compound, dose and cell line ($512$ arms, $1{,}536$ wells) is removed first and acts as the ground truth for computing outcome measures. Of the remainder, $20\%$ of wells ($47{,}800$ images) are held out for validation. The training set ranges from $22{,}076$ images at $\eta=0.1$ to $25{,}676$ at $\eta=1$.

We define the utility score now. A ResNet-18 with a five-channel input, fine-tuned to classify the disease condition, embeds each image as $\phi(Y)\in\R^{512}$, whitened by typical-variation normalization fitted on vehicle wells (from the validation set). 
Let $\mu_{\mathrm{mock}}$ and $\mu_{\mathrm{inf}}$ be the mean embeddings of held-out mock-infected and untreated infected wells, then define $G=\|\mu_{\mathrm{inf}}-\mu_{\mathrm{mock}}\|$ and $\hat u=(\mu_{\mathrm{inf}}-\mu_{\mathrm{mock}})/G$. Then
\[
    g(Y)=\frac{G-\langle\phi(Y)-\mu_{\mathrm{mock}},\hat u\rangle}{G},
\]
Above, $g=1$ is full rescue to mock, $g=0$ is no better than vehicle, and $g<0$ is worse. 

\paragraph{Biased logging.} In VERO, every compound at each dose $\geq 1\,\mu$M keeps $k=10\eta$ of its roughly $10$ training wells, with randomized rounding for non-integer $k$, and $\eta\in\{0.1,0.25,0.5,0.75,1\}$. 
HRCE and all lower doses are unaffected. 

\paragraph{Why this is hard.} High signal only occurs for three active compounds at high doses in VERO, and these are the interventions the logger avoids. The conditional generator, is subject to many null intervention compounds limiting its learning.

\paragraph{Models.} Generators are diffusion transformers (trained by flow matching over $200$ epoche under a pretrained KL-autoencoder (\texttt{ostris/vae-kl-f8-d16}, applied per fluorescence channel). They condition on the cell line, the compound (a learned embedding, with a token for untreated wells) and the log-dose.

\paragraph{Environment-aware training.} \Cref{fig:rx_inv} repeats the experiment with the generator also conditioned on the experiment batch (two per cell line). These runs also carry the penalty~\eqref{eq:inv} ($\lambda_\inv=0.3$). The retargeted generator remains superior in terms of regret.

\begin{figure}[ht!]
    \centering
    \includegraphics[width=\linewidth]{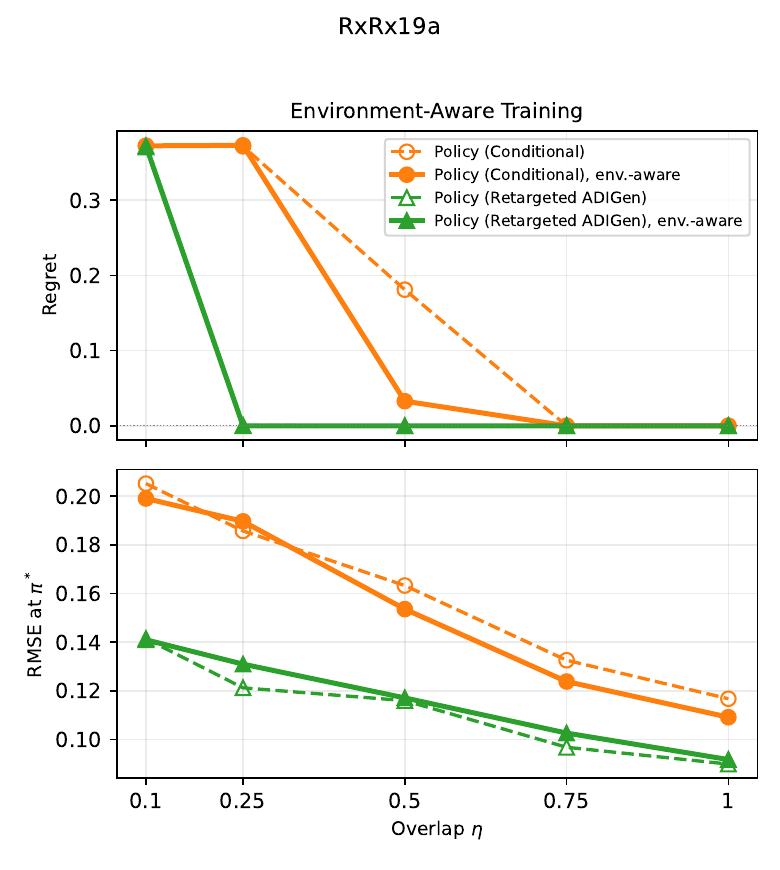}
    \caption{RxRx19a with environment-aware training, as overlap $\eta$ variesDashed, open markers: gent batch; solid, filled:generator conditioned on the experiment batch. Top: normalized regret;        environment-aware traininrator's regret at$\eta=0.5$ from $0.18$ to $0.03$, and the retargeted generator reaches zero   regret from $\eta=0.25$ i on held-out wells of theactive compounds at doses $\geq 1\,\mu$M; the retargeted generator stays more accurate than the conditi.}
    \label{fig:rx_inv}
\end{figure}

\paragraph{Dose Breakdown Exploration}
\Cref{fig:rx_dose} shows the mechanism behind the regret results. The real wells show sharp rescue at high doses, which the logger thins. Fit to the logged distribution, the conditional generator and \ADG{} reduce effects toward negligible effects that dominate the cell data. Retargeting reweights toward the high-dose interventions the policy favors and recovers a larger share of the signal, especially when overlap is low.
  \begin{figure}[ht!]
      \centering
      \includegraphics[width=0.9\linewidth]{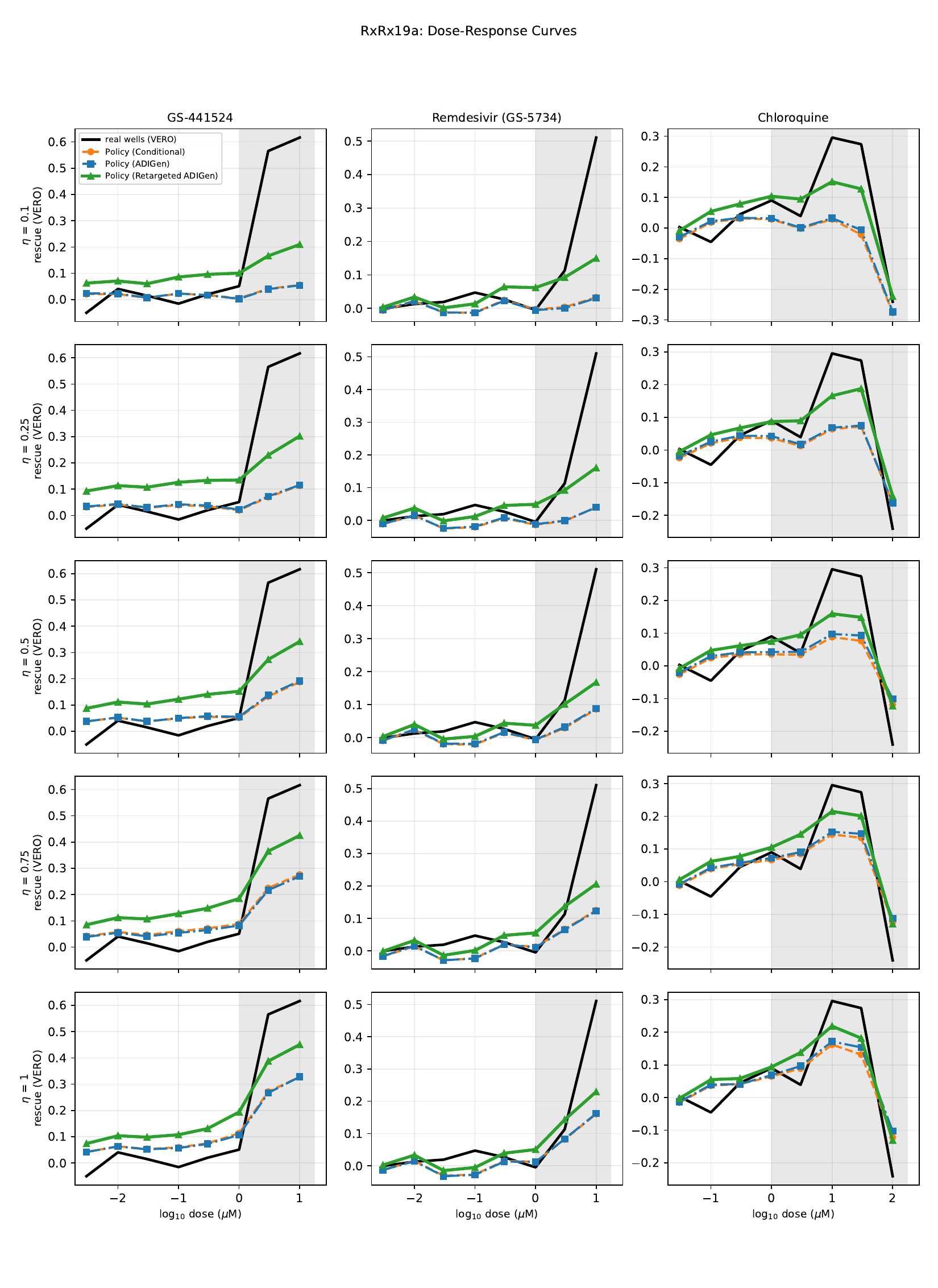}
        \caption{RxRx19a dose--response curves in VERO. Mean rescue score against $\log_{10}$ dose for the three active compounds, generated by each method and compared with the held-out reserve wells (black). Rows correspond to overlap $\eta\in\{0.1,0.25,0.5,0.75,1\}$; shaded doses ($\geq 1\,\mu$M) are thinned by the logging policy. The conditional generator and \ADG{} coincide and shrink the high-dose rescue toward zero; retargeted \ADG{} recovers more of it at every $\eta$, most visibly at low overlap, at the cost of a slight upward bias at inactive doses.}
      \label{fig:rx_dose}
  \end{figure}

\subsection{PhyWorld}\label{app:exp-video}
\paragraph{Data.} Each scene has $m\in{1,2,4,8}$ balls in fixed lanes, each of which under go 2 `kicks' or horizontal pushes. The first kick $a_{1j}=\pm 0.3$ is logged, followed by one common gust $g\sim\mathcal N(0,0.25^2)$, all visible in the model's $10$-frame conditioning window. The policy chooses the second kick $a\in\R^m$ (intervention dimension $d_A=m$), with kick noise of standard deviation $0.1$. 
The outcome is the rendered video, where the utility is the distance of each ball's landing position from a goal. We train on $n=3{,}000$ scenes per overlap level and on $n\in{3\text{k},6\text{k},10\text{k},30\text{k}}$ along $m$, matching data density per dimension. All arms are evaluated on $256$ out-of-distribution scenes. 

\paragraph{Biased logging.} The logger anticipates the gust,
\[
 a_j=-0.55(a_{1j}+g)+U(-0.08,0.08)
\]
clipped to the action box $[-0.45,0.45]$. With probability $\eta$ the whole joint action is instead drawn uniformly over the box. We vary $\eta \in{0.002,0.005,0.01,0.05}$ at $m=1$ and $m\in{1,2,4,8}$ at $\eta=0.005$.

\paragraph{Why this is hard.} Logged actions concentrate in a narrow band that cancels the gust, so logged actions that are able to hit the goal state appear only in rare draws of the data. As in the synthetic setting, the density ratio of the best policy against the logger grows with $m$.

\paragraph{Outcome model.} Generators are conditional diffusion transformers (about $3$M parameters) trained by in the latent space of a frozen pretrained Stable Diffusion VAE (\texttt{stabilityai/sd-vae-ft-mse}, $8\times$ spatial downsampling), applied frame by frame. They condition on the $10$ observed frames and the second kick, and the landing positions are read from the generated video.

\end{document}